\documentclass[11pt]{article}
\usepackage[utf8]{inputenc} 
\usepackage[T1]{fontenc}    
\usepackage{wrapfig}

\usepackage{setspace}
\usepackage{cprotect}
\usepackage{amsmath,amssymb,amsthm}
\usepackage[noend]{algorithmic}
\usepackage[ruled,vlined]{algorithm2e}
\usepackage{hyperref}
\usepackage{fullpage}
\usepackage{makeidx}
\usepackage{enumerate}
\usepackage{graphicx,float,psfrag,epsfig}
\usepackage{epstopdf}
\usepackage{color}
\usepackage{enumitem}
\usepackage{subfig}
\usepackage{caption}
\usepackage{bigints}
\usepackage{mathtools}
\usepackage[mathscr]{euscript}

\usepackage{amsfonts}
\usepackage{soul}
\usepackage{csquotes}
\usepackage{amsmath}
\usepackage{amsthm}
\usepackage{amsfonts}
\usepackage{hyperref}
\usepackage{comment}
\usepackage{graphicx}
\usepackage{xcolor}
\usepackage{parskip}
\usepackage{hyperref}       
\usepackage{url}            
\usepackage{booktabs}       
\usepackage{amsfonts}       
\usepackage{nicefrac}       
\usepackage{microtype}      
\usepackage{xcolor}         
\usepackage{microtype}
\usepackage{graphicx}
\usepackage{amsfonts}
\usepackage{amsmath}
\usepackage{amssymb}
\usepackage{bbm}
\usepackage{multirow}
\usepackage{mathtools}
\usepackage{relsize}

\def\P{\mathbb{P}}

\def\R{\mathbb{R}}

\newcommand{\opnorm}[1]{\left\lVert#1\right\rVert_{\textup{op}}}

\def\b0{{0}}

\def\RR{\mathbb{R}}

\def\>{\rangle}

\newcommand{\frf}{\Phi_{\textup{RF}}}
\newcommand{\tfrf}{\tilde \Phi_{\textup{RF}}}
\newcommand{\krf}{K_{\textup{RF}}}
\newcommand{\tkrf}{\tilde K_{\textup{RF}}}
\newcommand{\Irf}{\mathcal{I}_{\textup{RF}}}

\newcommand{\fntk}{\Phi_{\textup{NTK}}}
\newcommand{\kntk}{K_{\textup{NTK}}}

\newcommand{\E}{\mathbb{E}}

\newcommand{\distas}[1]{\mathbin{\overset{#1}{\sim}}}

\newcommand\ceil[1]{\lceil#1\rceil}

\newcommand{\bigO}[1]{\mathcal{O}\left(#1\right)}

\newcommand{\norm}[1]{\left\|#1\right\|}
\newcommand{\subGnorm}[1]{\left\|#1\right\|_{\psi_2}}
\newcommand{\subEnorm}[1]{\left\|#1\right\|_{\psi_1}}

\newcommand{\abs}[1]{\left|#1\right|}

\newcommand{\evmax}[1]{\lambda_{\rm max}\left(#1\right)}
\newcommand{\evmin}[1]{\lambda_{\rm min}\left(#1\right)}

\def\Lip{\mathrm{Lip}}
\def\op{\mathrm{op}}

\def\PP{\mathbb{P}}

\def\min{\mathop{\rm min}\nolimits}
\def\max{\mathop{\rm max}\nolimits}

\def\ie{\textit{i.e. }}

\numberwithin{equation}{section}

\newtheoremstyle{myexample} 
    {\topsep}                    
    {\topsep}                    
    {\rm }                   
    {}                           
    {\bf }                   
    {.}                          
    {.5em}                       
    {}  

\newtheoremstyle{myremark} 
    {\topsep}                    
    {\topsep}                    
    {\rm}                        
    {}                           
    {\bf}                        
    {.}                          
    {.5em}                       
    {}  

\newtheorem{claim}{Claim}[section]
\newtheorem{lemma}[claim]{Lemma}

\newtheorem{assumption}{Assumption}

\newtheorem{theorem}{Theorem}

\theoremstyle{myremark}

\theoremstyle{myremark}

\theoremstyle{myexample}

\author{Simone Bombari\thanks{Institute of Science and Technology Austria (ISTA). Emails: \texttt{\{simone.bombari, marco.mondelli\}@ist.ac.at}.}\;,
\;\;Shayan Kiyani\thanks{University of Pennsylvania. Email: \texttt{shayank@seas.upenn.edu}.}\;,
\;\;Marco Mondelli\footnotemark[1]}

\title{Beyond the Universal Law of Robustness:
Sharper Laws for Random Features and Neural Tangent Kernels}

\begin{document}

\maketitle

\begin{abstract}
Machine learning models are vulnerable to adversarial perturbations, and a thought-provoking paper by Bubeck and Sellke has analyzed this phenomenon through the lens of over-parameterization: interpolating smoothly the data requires significantly more parameters than simply memorizing it. However, this ``universal'' law provides only a necessary condition for robustness, and it is unable to discriminate between models. In this paper, we address these gaps by focusing on empirical risk minimization in two prototypical settings, namely, random features and the neural tangent kernel (NTK). We prove that, for random features, the model is not robust for any degree of over-parameterization, even when the necessary condition coming from the universal law of robustness is satisfied. In contrast, for even activations, the NTK model meets the universal lower bound, and it is robust as soon as the necessary condition on over-parameterization is fulfilled. This also addresses a conjecture in prior work by Bubeck, Li and Nagaraj. Our analysis decouples the effect of the kernel of the model from an ``interaction matrix'', which describes the interaction with the test data and captures the effect of the activation. Our theoretical results are corroborated by numerical evidence on both synthetic and standard datasets (MNIST, CIFAR-10). 
\end{abstract}

\section{Introduction}

Despite the 
deployment of deep neural networks in a wide set of applications, these models are known to be vulnerable to adversarial perturbations \cite{Biggio_2013, https://doi.org/10.48550/arxiv.1312.6199}, which raises serious concerns about their robustness guarantees. 
To address these issues, a wide variety of adversarial training methods has been developed 
\cite{goodfellowrob, kurakin2017adversarial, madry2018towards, raghunathan2018certified, pmlr-v80-wong18a}. In a parallel effort, a recent line of theoretical work has focused on providing a principled understanding to the phenomenon of adversarial robustness, see e.g. \cite{dohmatob,zhu2022robustness,pmlr-v125-javanmard20a,donhauser2021interpolation} and the review in Section \ref{sec:rel}.
Specifically, the recent paper by \cite{bubeck2021a} has highlighted the role of over-parameterization as a \emph{necessary} condition to achieve robustness: while interpolation requires the number of parameters $p$ of the model to be at least linear in the number of data samples $N$, \emph{smooth} interpolation (namely, robustness) requires $p > dN$, where $d$ is the input dimension. We note that the law of robustness put forward by \cite{bubeck2021a} is ``universal'' in the sense that it holds for \emph{any} parametric model. Furthermore, an earlier conjecture by \cite{bubeck2021law} posits that there exists a model of two-layer neural network such that robustness is successfully achieved with this minimal number of parameters; i.e., the condition $p > dN$ is both \emph{necessary} and \emph{sufficient} for robustness. This state of affairs leads to the following natural question:
\vspace{0.5em}
\begin{center}
    \textit{When does over-parameterization become a sufficient condition to achieve adversarial robustness?}
\end{center}
\vspace{0.5em}
In this paper, we consider a \emph{sensitivity}\footnote{The sensitivity can be interpreted as the non-robustness.} measure proportional to $\norm{\nabla_z f(z, \theta)}_2$, where $z$ is the test data point and $f$ is the model with parameters $\theta$ whose robustness is investigated. This quantity is closely related to notions of sensitivity appearing in related works \cite{dohmatob,dohmatob2022non,bubeck2021a}, and it leads to a more stringent requirement on the robustness than the perturbation stability considered by \cite{zhu2022robustness}, see also the discussion in Section \ref{sec:prel}. For such a sensitivity measure, we show that the answer to the question above depends on the specific model at hand. Specifically, we will focus on the solution yielded by empirical risk minimization (ERM) in two prototypical settings widely analyzed in the theoretical literature: \emph{(i)} Random Features (RF) \cite{rahimi2007random}, and the \emph{(ii)} Neural Tangent Kernel (NTK) \cite{JacotEtc2018}. 

\paragraph{Main contribution.}\hspace{-1em} Our key results are summarized below.

\begin{itemize}
    \item For random features, we show that ERM leads to a model which is \emph{not robust} for \emph{any} degree of over-parameterization. Specifically, we tackle a regime in which the universal law of robustness by \cite{bubeck2021a} trivializes, and we provide a more refined bound.

    \item  For NTK with an even activation function, we give an upper bound on the ERM solution that \emph{matches} the lower bound 
    by \cite{bubeck2021a}. 
    As the NTK model approaches the behavior of gradient descent for a suitable initialization \cite{chizat2019lazy}, this also shows that a class of two-layer neural networks has sensitivity of order $\sqrt{Nd/p}$. This addresses Conjecture 2 of \cite{bubeck2021law}, albeit for a slightly different notion of sensitivity (see the comparison in Section \ref{sec:prel}). 
\end{itemize}

At the technical level, our analysis involves the study of the spectra of RF and NTK random matrices, and it provides the following insights.

\begin{itemize}
    \item We introduce in \eqref{eq:interactionmatrix} a new quantity, dubbed the \emph{interaction matrix}. Studying this matrix in isolation is a key step in all our results, as its different norms are intimately related to the sensitivity of both the RF and the NTK model. To the best of our knowledge, this is the first time that attention is raised over such an object, which we deem relevant to the theoretical characterization of adversarial robustness.

    \item Our analytical characterization captures the role of the activation function: it turns out that a specific symmetry (e.g., being even) boosts the adversarial robustness of the models.
\end{itemize}

Finally, we experimentally show that the robustness behavior of both synthetic and standard datasets (MNIST, CIFAR-10) agrees well with our theoretical results.

\section{Related work}\label{sec:rel}




From the vast literature studying adversarial examples, we succintly review related works focusing on linear models, random features and NTK models. 


\paragraph{Linear regression.} In light of the universal law of robustness, in the interpolation regime, linear regression cannot achieve adversarial robustness (as $p = d$). For linear models, \cite{donhauser2021interpolation,pmlr-v125-javanmard20a} have focused on adversarial training (as opposed to ERM, which is considered in this work), showing that over-parameterization hurts the robust generalization error. We also point out that, in some settings, even the standard generalization error is maximized when the model is under-parametrized \cite{hastie2022surprises}. Precise asymptotics on the robust error in the classification setting are provided in \cite{javanmard2022precise,taheri2020asymptotic}. 
  A different approach is pursued by \cite{tsipras2018robustness}, who study a linear model that exhibits a trade-off between generalization and robustness. 


\paragraph{Random features (RF).} The RF model introduced by \cite{rahimi2007random} can be regarded as a two-layer neural network with random first layer weights, and it solves the lack of freedom in the number of trainable parameters, which is constrained to be equal to the input dimension for linear regression. The popularity of random features derives from their analytical tractability, combined with the fact that the model still reproduces behaviors typical of deep learning, such as the double-descent curve \cite{mei2022generalization}. Our analysis crucially relies on the spectral properties of the kernel induced by this model. Such properties have been studied by \cite{rfspectrum}, and tight bounds on the smallest eigenvalue of the feature kernel have been provided by \cite{tightbounds} for ReLU activations. Theoretical results on the adversarial robustness of random features have been shown for both adversarial training \cite{hamedadversarial} and the ERM solution \cite{dohmatob, dohmatob2022non}. We will discuss the comparison with \cite{dohmatob, dohmatob2022non} in the next paragraph.
 

\paragraph{Neural tangent kernel (NTK).} The NTK can be regarded as the kernel obtained by linearizing the neural network around the initialization \cite{JacotEtc2018}. A popular line of work has analyzed its spectrum \cite{fan2020spectra,adlam2020neural,wang2021deformed} and bounded its smallest eigenvalue \cite{theoreticalinsghts,tightbounds,montanari2022interpolation,bombari2022memorization}. The behavior of the NTK is closely related to memorization \cite{montanari2022interpolation}, optimization \cite{AllenZhuEtal2018,DuEtal2019} and generalization \cite{arora2019fine} properties of deep neural networks. 
The NTK has also been recently exploited to identify non-robust features learned with standard training \cite{tsilivis2022what} and gain insight on adversarial training \cite{loo2022evolution}.
More closely related to our work, the robustness of the NTK model is studied in \cite{zhu2022robustness,dohmatob2022non,dohmatob}. Specifically, \cite{zhu2022robustness} analyze the interplay between width, depth and initialization of the network, thus improving upon results by \cite{huang2021exploring,wu2021wider}. However, the perturbation stability considered by \cite{zhu2022robustness} captures an average robustness, rather than an adversarial one (see the detailed comparison in Section \ref{sec:prel}), which makes these results not immediately comparable to ours. In contrast, our notion of sensitivity is close to that analyzed in \cite{dohmatob2022non,dohmatob}, which establish trade-offs between robustness and memorization for both the RF and the NTK model. However, \cite{dohmatob2022non} study a teacher-student model with quadratic teacher and infinite data, which also does not allow for a direct comparison. Finally, \cite{dohmatob} studies the ERM solution and focuses on \emph{(i)} infinite-width networks, and \emph{(ii)} finite-width networks where the number of neurons scales linearly with the input dimension. We remark that this last setup does not lead to lower bounds on the sensitivity tighter than those by \cite{bubeck2021a}, and our results on the NTK model provide the first upper bounds on the sensitivity (which also match the lower bounds by \cite{bubeck2021a}).


\section{Preliminaries}\label{sec:prel}

\paragraph{Notation.} 
Given a vector $v$, we indicate with $\norm{v}_2$ its Euclidean norm.  Given $v \in \R^{d_v}$ and $u \in \R^{d_u}$, we denote by $v \otimes u \in \R^{d_v d_u}$ their Kronecker product. Given a matrix $A$, let $\norm{A}_{\op}$ be its operator norm, $\norm{A}_{F}$ its Frobenius norm and $\evmin{A}$ and $\evmax{A}$ its smallest and largest eigenvalues respectively. We denote by $\norm{\varphi}_{\Lip}$ the Lipschitz constant of the function $\varphi$.  All the complexity notations $\Omega(\cdot)$, $\mathcal{O}(\cdot)$, $o(\cdot)$ and $\Theta(\cdot)$ are understood for sufficiently large data size $N$, input dimension $d$, number of neurons $k$, and number of parameters $p$. 
We indicate with $C,c>0$ numerical constants, independent of $N, d, k, p$. 


\paragraph{Generalized linear regression.}
Let $(X, Y)$ be a labelled training dataset, where $X=[x_1, \ldots, x_N]^\top \in \R^{N \times d}$ contains the training samples on its rows 
and $Y=(y_1, \ldots, y_N) \in \R^N$ contains the corresponding labels. 
Let $\Phi : \R^d \to \R^p$ be a generic \emph{feature map}, from the input space to a feature space of dimension $p$. We consider the following \emph{generalized linear model} 
\begin{equation}\label{eq:glm}
    f(x, \theta) = \Phi(x)^\top \theta,
\end{equation}
where $\Phi(x) \in \R^p$ is the feature vector associated with the input sample $x$, and $\theta \in \R^p$ is 
the set of the trainable parameters of the model. Our supervised learning setting involves solving the 
optimization problem
 \begin{equation}\label{eq:optloss}
    \min_\theta \norm{\Phi(X)^\top \theta - Y}_2^2,
\end{equation}
where $\Phi(X) \in \R^{N \times p}$ is the feature matrix, containing $\Phi (x_i)$ in its $i$-th row.
From now on, we use the shorthands $\Phi := \Phi(X)$ and $K := \Phi \Phi^\top \in \R^{N \times N}$, where $K$ denotes the kernel associated with the feature map. If we assume $K$ to be invertible (i.e., the model is able to fit any set of labels $Y$), it is well known that gradient descent converges to the interpolator which is the closest in $\ell_2$ norm to the initialization, see e.g. \cite{smallnorm}. In formulas, 
\begin{equation}\label{eq:thetastar}
    \theta^* = \theta_0 + \Phi^\top K^{-1} (Y - f(X, \theta_0)),
\end{equation}
where $\theta^*$ is the gradient descent solution, $\theta_0$ is the initialization and 
$f(X, \theta_0) = \Phi(X)^\top \theta_0$ is the output of the model \eqref{eq:glm} at initialization.  




\paragraph{Sensitivity.}
We measure the robustness of a model $f(z, \theta)$ as a function of the test sample $z$. In particular, we are interested in 
a quantity that expresses how \emph{sensitive} is the output when small perturbations are applied to the input $z$. More specifically, we could imagine an adversarial example \enquote{built around} $z$ as 
\begin{equation}\label{eq:lessless}
    z_{\textup{adv}} = z + \Delta_{\textup{adv}}, \quad\mbox{with }\norm{\Delta_{\textup{adv}}}_2 \leq \delta \norm{z}_2.
\end{equation}
Here, $\Delta_{\textup{adv}}$ is a \emph{small} adversarial perturbation, in the sense that its $\ell_2$ norm is only a fraction $\delta$ of the $\ell_2$ norm of $z$. Crucially, we expect $\delta$ not to depend on the scalings of the problem. For example, if $z$ represents an image with $d$ pixels, this can correspond to perturbing every pixel by a given amount. In this case, the robustness depends on the amount of perturbation per pixel, and not on the number of pixels $d$. This in turn implies that $\delta$ is a numerical constant independent of $d$.



We are interested in the response of the model to $z_{\textup{adv}}$, and how this relates to the natural scaling of the output, which we assume to be $\Theta(1)$.\footnote{This is a natural choice, as e.g. the output label is not expected to grow with the input dimension. However, any scaling is in principle allowed and would not change our final results. \label{foot:outscaling}} Up to first order, we can write
\begin{equation}\label{eq:firstord}
\begin{aligned}
    \left| f(z_{\textup{adv}}, \theta) - f(z, \theta) \right| &\simeq \left| \nabla_z f(z, \theta)^\top  \Delta_{\textup{adv}} \right|\\
    & \leq \norm{\Delta_{\textup{adv}}} \norm{\nabla_z f(z, \theta)}_2 \\
    & \leq\delta \norm{z}_2 \norm{\nabla_z f(z, \theta)}_2.
\end{aligned}
\end{equation}
The first inequality is saturated if the attacker builds an adversarial perturbation that perfectly aligns with $\nabla_z f(z, \theta)$, and the second inequality follows from \eqref{eq:lessless}. Hence, we 
define the \emph{sensitivity} of the model evaluated in $z$ as
\begin{equation}\label{eq:sensitivity}
\mathcal S_{f_{\theta}}(z) := \norm{z}_2 \norm{\nabla_z f(z, \theta)}_2.
\end{equation}
Recall that both $\delta$ and the output of the model are constants, independent of the scaling of the problem. Hence, a robust model needs to respect $\mathcal S_{f_{\theta}}(z) = \bigO{1}$ for its test samples $z$. In contrast, if $\mathcal S_{f_{\theta}}(z) \gg 1$ (e.g., the sensitivity grows with the number of pixels of the image), we expect the model to be adversarially vulnerable when evaluated in $z$. In a nutshell, the goal of this paper is to provide bounds on $\mathcal S_{f_{\theta}}(z)$ for random features and NTK models, thus establishing their robustness.


 
\paragraph{Related notions of sensitivity.} Measures of robustness similar to 
\eqref{eq:sensitivity} have been used  
in the related literature. In particular, \cite{dohmatob2022non} study $\E_{z \sim P_X}[\mathcal S^2_{f_{\theta}}(z)]$ in the context of a teacher-student setting with a quadratic target. \cite{dohmatob} considers the model obtained from solving the generalized regression problem \eqref{eq:optloss} and characterizes a sensitivity measure given by the Sobolev semi-norm. This quantity is again similar to $\E_{z \sim P_X}[\mathcal S^2_{f_{\theta}}(z)]$, the only difference being that the gradient is projected onto the sphere (namely, onto the data manifold) before taking the norm. \cite{zhu2022robustness} study a notion of perturbation stability given by $\E_{z, \hat z, \theta}\left| \nabla_z f(z, \theta)^\top (z-\hat z) \right|$, where $z$ is sampled from the data distribution and $\hat z$ is uniform in a ball centered at $z$ with radius $\delta$. This would be similar to averaging over $\Delta_{\textup{adv}}$ in \eqref{eq:firstord}, instead of choosing it adversarially. Hence, \cite{zhu2022robustness} capture a weaker notion of \emph{average robustness}, as opposed to this work which studies the \emph{adversarial robustness}. 
Finally, \cite{bubeck2021law,bubeck2021a} consider $\norm{f(z, \theta)}_{\text{Lip}}$, where the Lipschitz constant is with respect to $z$, which is equivalent to $\max_z\mathcal S_{f_{\theta}}(z) / \norm{z}_2$. While this is a more stringent condition than \eqref{eq:sensitivity}, we remark that the adversarial robustness corresponds to choosing adversarially $\Delta_{\textup{adv}}$ (and not $z$), as done in \eqref{eq:firstord}. We also note that our main results will hold with high probability over the distribution of $z$.

To conclude, we remark that, differently from previous work \cite{bubeck2021a, dohmatob}, we opt for a \emph{scale invariant} definition of sensitivity. This derives from the term $\norm{z}_2$ appearing in the RHS of \eqref{eq:sensitivity}, and it allows us to apply the definition to data with arbitrary scaling. In particular, if we set $\norm{z}_2=1$, then \eqref{eq:lessless} would yield $\norm{\Delta_{\rm adv}}_2 \leq \delta$. Both \cite{bubeck2021a} and \cite{dohmatob} consider this scaling of the data: the former paper uses $\bigO{1}$ Lipschitz constant as an indication of a robust model, and the latter considers a term similar to $\norm{\nabla_z f(z, \theta)}_2$.


\paragraph{Robustness of generalized linear regression.}
For 
generalized linear models with feature map $\Phi$, plugging \eqref{eq:glm} 
in \eqref{eq:sensitivity} gives 
\begin{equation}
    \mathcal S_{f_{\theta}}(z) = \norm{z}_2 \norm{\nabla_z \Phi(z)^\top \theta}_2.
\end{equation}
This expression can also provide the sensitivity of the model trained with gradient descent. Let us assume that $f(x, \theta_0) = 0$ for all $x \in \R^d$ and, as before, that the kernel of the feature map is invertible. Then, by plugging the value of $\theta^*$ from \eqref{eq:thetastar} in the previous equation, 
we get
\begin{equation}\label{eq:sensforlin}
    \mathcal S_{f_{\theta^*}}(z) = \norm{z}_2 \norm{\nabla_z \Phi(z)^\top \Phi^\top K^{-1} Y}_2.
\end{equation}


\section{Main Result for Random Features}\label{sec:rf}

In this section, we provide our law of robustness for the \emph{random features} model. In particular, we will show that, for a wide class of activations, the sensitivity of random features is $\gg 1$ and, therefore, the \emph{model is not robust}.

We assume the data labels $Y \in \R^N$ to be given by $Y = g(X) + \epsilon$. Here, $g: \R^d \to \R$ is the ground-truth function applied row-wise to $X\in \mathbb R^{N\times d}$ and $\epsilon \in \R^N$ is a noise vector independent of $X$, having independent sub-Gaussian entries with zero mean and variance $\E[\epsilon_{i}^2] = \varepsilon^2$ for all $i$. The \emph{random features (RF) model} takes the form
\begin{equation}
    f_{\textup{RF}}(x, \theta) = \frf(x)^\top \theta, \qquad \frf(x)=\phi(V x),
\end{equation}
where $V$ is a $k \times d$ matrix, such that $V_{i,j} \distas{}_{\rm i.i.d.}\mathcal{N}(0, 1 / d)$, and $\phi$ is an activation function applied component-wise. The number of parameters of this model is $k$, as $V$ is fixed and $\theta\in \mathbb R^k$ contains the trainable parameters. We initialize the model at $\theta_0 = 0$ so that $f_{\textup{RF}}(x, \theta_0) = 0$ for all $x$. After training, we can write the sensitivity using \eqref{eq:sensforlin}, which gives
\begin{equation}\label{eq:robustnessrf}
\begin{aligned}
    \mathcal S_{\textup{RF}}(z) 
    &= \norm{z}_2 \norm{\nabla_z \frf(z)^\top \frf^\top \krf^{-1} Y}_2,
\end{aligned}
\end{equation}
where we use the shorthands $\frf := \frf(X) = \phi(XV^\top)\in \R^{N \times k}$ for the RF map and $\krf := \frf \frf^\top \in \R^{N \times N}$ for the RF kernel. We will show that the kernel is in fact invertible in the argument of our main result, Theorem \ref{thm:rfnorobust}. Throughout this section, we make the following assumptions.

\begin{assumption}[Data distribution]\label{ass:datadist}
        The input data $(x_1,\ldots,x_N)$ are $N$ i.i.d.\ samples from the distribution $P_X$, which satisfies the following properties:
	\begin{enumerate}
		\item $\norm{x}_2 = \sqrt{d}$, i.e., the data are distributed on the sphere of radius $\sqrt d$.
		\item $\E[x] = 0$, i.e., the data are centered.
    	\item $P_X$ satisfies the \emph{Lipschitz concentration property}. Namely, there exists an absolute constant $c>0$ such that, for every Lipschitz continuous function $\varphi: \RR^d \to \RR$, we have
        for all $t>0$,
	\begin{equation}
    	    \P\left(\abs{\varphi(x)- \E_X [\varphi(x)]}>t\right) \leq 2e^{-ct^2 / \norm{\varphi}_{\Lip}^2}.
    	\end{equation}
	\end{enumerate}
\end{assumption}


The first two assumptions can be achieved by simply pre-processing the raw data. For the third assumption, we remark that the family of Lipschitz concentrated distributions covers a number of important cases, e.g., standard Gaussian \cite{vershynin2018high}, uniform on the sphere and on the unit (binary or continuous) hypercube \cite{vershynin2018high}, or data obtained via a Generative Adversarial Network (GAN) \cite{seddik2020random}. The third requirement is common in the related literature \cite{tightbounds, bombari2022memorization} and it is equivalent to the isoperimetry condition required by \cite{bubeck2021a}. We remark that the three conditions of Assumption \ref{ass:datadist} are often replaced by a stronger requirement 
(e.g., data uniform on the sphere), see \cite{montanari2022interpolation, dohmatob}.

Finally, we note that our choice of data scaling is different from the related work \cite{bubeck2021a, dohmatob}, as they both consider $\norm{z}_2=1$. We choose $\norm{z}_2=\sqrt{d}$, as we believe it to be more natural in practice. Consider, for example, the case in which $z$ is an image and $d$ is the number of pixels. If the pixel values are renormalized in a fixed range, then $\norm{z}_2$ scales as $\sqrt{d}$. This choice shouldn't worry the reader, as our definition of sensitivity \eqref{eq:sensitivity} is scale invariant, facilitating the comparison between previous results and ours.



\begin{assumption}[Activation function]\label{ass:activationfunc}
    The activation function $\phi$ satisfies the following properties:
    \begin{enumerate}
    \item $\phi$ is a non-linear 
    $L$-Lipschitz function.
    \item Its first and second order derivatives $\phi'$ and $\phi''$ are respectively $L_1$ and $L_2$-Lipschitz functions.
    \end{enumerate}
\end{assumption}
These requirements are satisfied by common activations, e.g. smoothed ReLU, sigmoid, or $\tanh$.

\begin{assumption}[Over-parameterization]\label{ass:overparam}
 \begin{align}
        & N \log^3 N = o(k),\label{eq:overparam}\\
        & d \log^4 d = o(k). 
\end{align}
\end{assumption}


In words, the number of neurons $k$ scales faster than the input dimension $d$ and the number of data points $N$. Such an over-parameterized setting allows to \emph{(i)} perfectly interpolate the data (this follows directly from the invertibility of the kernel, which can be readily deduced from the proof of Theorem \ref{thm:rfnorobust}), and \emph{(ii)} achieve minimum test error, see \cite{mei2022generalization}. This is also in line with the current trend of heavily over-parametrized deep-learning models \cite{Nakkiran2020Deep}.

\begin{figure}[!t]
  \begin{center}
    \includegraphics[width=0.99\textwidth]{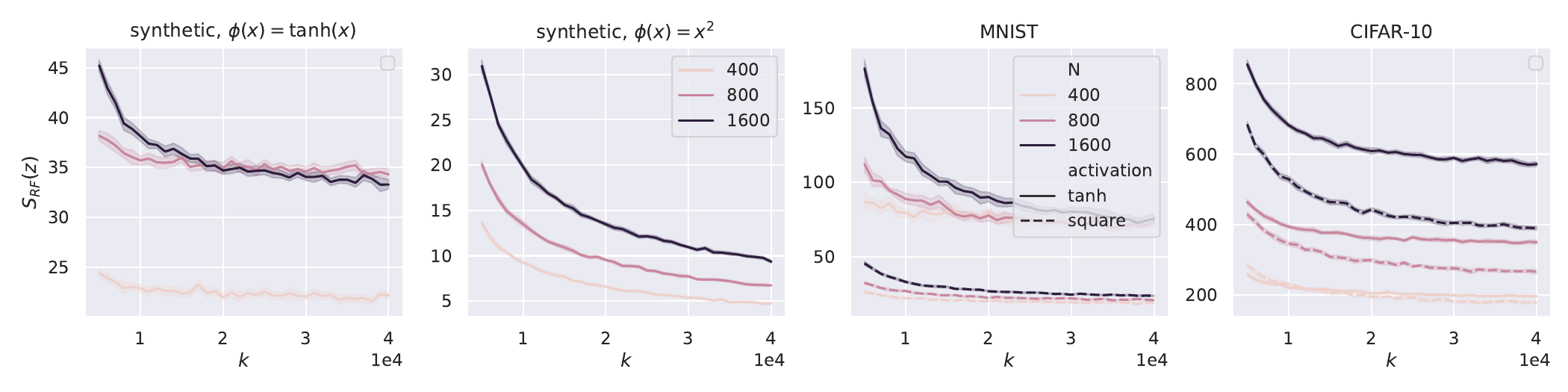}
  \end{center}
  \caption{Sensitivity as a function of the number of parameters $k$, for an RF model with synthetic data sampled from a Gaussian distribution with input dimension $d=1000$ (first and second plot), and with inputs taken from two classes of the MNIST and CIFAR-10 datasets (third and fourth plot respectively). Different curves correspond to different values of the sample size $N\in \{400, 800, 1600\}$ and to different activations ($\phi(x) = \tanh(x)$ or $\phi(x) = x^2$). We plot the average over $10$ independent trials and the confidence interval at $1$ standard deviation. The values of the sensitivity are significantly larger for $\phi(x) = \tanh(x)$ than for $\phi(x) = x^2$.}
  \label{fig:rf_old_act}
\end{figure}

\begin{assumption}[High-dimensional data]\label{ass:dlarge}
 \begin{align}
       & N \log^3 N = o\left( d^{3/2}\right),\label{eq:dlarge}\\
      & \log^8 k = o(d).
    \end{align}
\end{assumption}


While the second condition is rather mild, the first one lower bounds the input dimension $d$ in a way that appears to be crucial to prove our main Theorem \ref{thm:rfnorobust}. Understanding whether \eqref{eq:dlarge} 
can be relaxed is left as an open question. 


At this point, we are ready to state our main result for random features.

\begin{theorem}\label{thm:rfnorobust}
Let Assumptions \ref{ass:datadist}, \ref{ass:activationfunc}, \ref{ass:overparam} and \ref{ass:dlarge} hold, and let $Y = g(X) + \epsilon$, such that the entries of $\epsilon$ are sub-Gaussian with zero mean and variance $\E[\epsilon_{i}^2] = \varepsilon^2$ for all $i$, independent between each other and from everything else. Let $z \sim P_X$ be a test sample independent from the training set $(X, Y)$, and let the derivative of the activation function satisfy $\E_{\rho\sim \mathcal N (0, 1)} [\phi' (\rho)] \neq 0$. Define $\mathcal S_{\textup{RF}}(z)$ as in \eqref{eq:robustnessrf}. Then,
\begin{equation}\label{eq:rfnr}
    \mathcal S_{\textup{RF}}(z) = \Omega \left( \sqrt[6]{N} \right) \gg 1,
\end{equation}
with probability at least $1 - \exp (- c \log^2 N)$ over $X$, $z$, $V$ and $\epsilon$.
\end{theorem}

In words, Theorem \ref{thm:rfnorobust} shows that a vast class of over-parameterized RF models is not robust against adversarial perturbations. A few remarks are now in order.


\paragraph{Beyond the universal law of robustness.}
We recall that the results by \cite{bubeck2021a} give a lower bound on the sensitivity of order $\sqrt{Nd/k}$. A lower bound of the same order is obtained by \cite{dohmatob}, whose results in the finite-width setting require $N, d, k$ to follow a proportional scaling (i.e., $N=\Theta(d)=\Theta(k)$). This leaves as an open question the characterization of the robustness for sufficiently over-parameterized random features (i.e., when $k\gg Nd$).  
Our Theorem \ref{thm:rfnorobust} resolves this question by showing that the RF model is in fact not robust for \emph{any} degree of over-parameterization (as long as $k\gg N, d$ and Assumption \ref{ass:dlarge} is satisfied).

\paragraph{Impact of the activation function and numerical results.} For the result of Theorem \ref{thm:rfnorobust} to hold, we require the additional assumption $\E_{\rho\sim \mathcal N (0, 1)} [\phi' (\rho)] \neq 0$. This may not be just a technical requirement. In fact, if $\E_{\rho\sim \mathcal N (0, 1)} [\phi' (\rho)] = 0$, then a key term appearing in the lower bound for the sensitivity (i.e., the Frobenius norm of the \emph{interaction matrix} defined in \eqref{eq:interactionmatrix}) has a drastically different scaling, see Theorem \ref{thm:int} in the proof outline below. The importance of the condition $\E_{\rho\sim \mathcal N (0, 1)} [\phi' (\rho)] \neq 0$ is also confirmed by our numerical simulations in the synthetic setting considered in Figure \ref{fig:rf_old_act}. The first plot (corresponding to $\phi(x)=\tanh(x)$ s.t. $\E_{\rho\sim \mathcal N (0, 1)} [\phi' (\rho)] \neq 0$) displays values of the sensitivity which are significantly larger than those in the second plot (corresponding to $\phi(x)=x^2$ s.t. $\E_{\rho\sim \mathcal N (0, 1)} [\phi' (\rho)] = 0$).
Furthermore, the sensitivity for $\phi(x)=\tanh(x)$ appears to reach a plateau for large values of $k$, while it keeps decreasing in $k$ when $\phi(x)=x^2$.
A similar impact of the activation function can be observed when considering the standard datasets MNIST and CIFAR-10 (third and fourth plot respectively).\footnote{The code used to obtain the results in Figures \ref{fig:rf_old_act}-\ref{fig:ntk_old_act} is available at the GitHub repository \href{https://github.com/simone-bombari/beyond-universal-robustness}{\texttt{https://github.com/simone-bombari/beyond-universal-robustness}}.} Additional experiments with different activation functions can be found in Appendix \ref{app:exp}.





\subsection{Outline of the Argument of Theorem \ref{thm:rfnorobust}}


The proof of Theorem \ref{thm:rfnorobust} can be divided into three steps. It will be convenient to define the shorthand
\begin{equation}\label{eq:Azrf}
A(z) := \nabla_z \frf(z)^\top \frf^\top \krf^{-1} \in \R^{d \times N}.
\end{equation}


\paragraph{Step 1. Fitting the noise lower bounds the sensitivity.}

First, we lower bound the sensitivity of our trained model with a quantity which does not depend on the labels and grows proportionally with the noise. 
\begin{theorem}\label{thm:noiselowerbound}
Let Assumptions \ref{ass:datadist}, \ref{ass:activationfunc}, \ref{ass:overparam} and \ref{ass:dlarge} hold. Define $A(z)$ as in \eqref{eq:Azrf}.
Then, we have
\begin{equation}\label{eq:thnoise}
    \mathcal S_{\textup{RF}}(z) \geq \frac{\varepsilon}{2} \|z\|_2 \norm{A(z)}_F,
\end{equation}
with probability at least $1 - \exp \left( -c \norm{A(z)}_F^2 / \opnorm{A(z)}^2 \right)$ over $\epsilon$, where $c$ is a numerical constant.
\end{theorem}

The proof exploits the independence between $g(X)$ and $\epsilon$ and the Hanson-Wright inequality. The details are contained in Appendix \ref{proof:noiselowerbound}.

Theorem \ref{thm:noiselowerbound} removes the labels from the expression, and reduces the problem of estimating the robustness to 
characterizing the Frobenius norm of 
$A(z)$ (as long as 
$\norm{A(z)}_F\gg \opnorm{A(z)}$, so that \eqref{eq:thnoise} holds with high probability). 
This is in line with \cite{bubeck2021a, dohmatob}, which provide upper bounds on the robustness of models fitting the data below noise level. 



\paragraph{Step 2. Splitting between interaction and kernel components.}
Next, we split the matrix $A(z)$ into two separate objects. The first is the \textit{interaction matrix} given by 
\begin{equation}\label{eq:interactionmatrix}
    \Irf(z) := \nabla_z \frf(z)^\top \tfrf^\top,
\end{equation}
where we use the shorthand $\tfrf := \frf - \E_X [\frf] \in \R^{N \times k}$. The second is the \emph{centered kernel} $\tkrf := \tfrf \tfrf^\top$, whose spectrum can be studied separately. 

\begin{theorem}\label{thm:Az}
    Let Assumptions \ref{ass:datadist}, \ref{ass:activationfunc}, \ref{ass:overparam} and \ref{ass:dlarge} hold. Define $A(z)$ as in \eqref{eq:Azrf}. Then, we have
    \begin{equation}\label{eq:Az1}
    \begin{aligned}
        & \norm{A(z)}_F \geq \lambda_{\max}^{-1}\left(\tkrf \right) \norm{\Irf(z)}_F - C \sqrt{N + d}/d, \\[4pt]
        & \norm{A(z)}_F \leq \lambda_{\min}^{-1}(\tkrf) \norm{\Irf(z)}_F + C \sqrt{N + d}/d,
    \end{aligned}
    \end{equation}
    which implies
    \begin{equation}\label{eq:Az2}
    \begin{aligned}
        & \norm{A(z)}_F \geq C_1 \frac{d}{kN}  \norm{\Irf(z)}_F - C \sqrt{N + d}/d,  \\[4pt]
        & \norm{A(z)}_F \leq C_2 k^{-1}  \norm{\Irf(z)}_F + C \sqrt{N + d}/d,
    \end{aligned}
    \end{equation}
    with probability at least $1 - \exp(-c \log^2 N)$ over $X$ and $V$, where $c$, $C$, $C_1$ and $C_2$ are absolute constants.
\end{theorem}

\emph{Proof sketch.} To prove the claim, first we characterize the extremal eigenvalues of the kernel $\krf:=\frf\frf^\top$ and of its centered counterpart $\tkrf:=\tfrf\tfrf^\top$, with $\tfrf=\frf-\E_X\frf$, see Appendix \ref{app:spectrum}. Then, we perform a delicate centering step on the matrix $A(z)$ defined in \eqref{eq:Azrf}. Informally, we show that 
\begin{equation}\label{eq:apprxo}
   \|\nabla_z \frf(z)^\top \frf^\top \krf^{-1}\|_F\approx \|\nabla_z \frf(z)^\top \tfrf^\top \tkrf^{-1}\|_F,
\end{equation}
see Lemma \ref{lemma:totalcentering} for a precise statement. 
This is the key technical ingredient, since the rank-one components coming from the average feature matrix have a large operator norm, which would trivialize the lower bound on the sensitivity. 
More specifically, the centering leads to two rank-one terms which are successively removed via the Sherman-Morrison formula and a number of ad-hoc estimates, see Appendix \ref{app:centering}. Finally, the proof of \eqref{eq:Az1}-\eqref{eq:Az2} follows by combining \eqref{eq:apprxo} with the bounds on the smallest/largest eigenvalues of $\tkrf$, and it appears at the end of Appendix \ref{app:centering}. \qed

Theorem \ref{thm:Az} unveils the crucial role played by the interaction matrix $\Irf(z)$ on the robustness of the model. This term is the only one containing the test sample $z$, and it depends on how the term $\nabla_z \frf(z)$ \textit{aligns} (or \enquote{interacts}) with the centered feature map of the training data $\tfrf$.




\paragraph{Step 3. Estimating $\norm{\Irf(z)}_F$.}

Finally, we provide a precise estimate on the norm of the interaction matrix $\norm{\Irf(z)}_F$. To do so, we assume that the test point $z$ is independently sampled from the data distribution $P_X$. 

\begin{theorem}\label{thm:int}
    Let Assumptions \ref{ass:datadist}, \ref{ass:activationfunc}, \ref{ass:overparam} and \ref{ass:dlarge} hold. Let $z \sim P_X$ be sampled independently from the training set $(X, Y)$, and $\Irf(z)$ be defined as in \eqref{eq:interactionmatrix}. 
    Then,
    \begin{equation}\label{eq:thminteq}
        \left| \norm{\Irf(z)}_F - \E_{\rho}^2 [\phi' (\rho)] \frac{k \sqrt N}{\sqrt d} \right| = o \left( \frac{k \sqrt N}{\sqrt d} \right),
    \end{equation}
    with probability at least $1 - \exp (- c \log^2 k)$ over $X$, $z$ and $V$, where $c$ is an absolute constant.
\end{theorem}

\emph{Proof sketch.} A direct calculation gives that 
\begin{equation*}
  \norm{\Irf(z)}_F^2 =  \sum_{i=1}^N \norm{V^\top \text{diag} \left( \phi' (Vz)\right) \tilde \phi (Vx_i)}_2^2,
\end{equation*}
and we bound separately each term of the sum. Note that the three terms $V^\top$, $\text{diag} \left( \phi' (Vz)\right)$ and $\tilde \phi (Vx_i)$ are correlated by the presence of $V$. However, $V$ contributes to the last two terms only via a single projection (along the direction of $z$ and $x_i$, respectively). Hence, the key idea is to use a Taylor expansion to split $\text{diag} \left( \phi' (Vz)\right)$ and $\tilde \phi (Vx_i)$ into a component correlated with $V$ and an independent one. The correlated components are computed exactly and the remaining independent term is shown to have a negligible effect. The details are in Appendix \ref{app:estimateI}. 
\qed

At this point, the proof of Theorem \ref{thm:rfnorobust} follows by combining the results of Theorems \ref{thm:noiselowerbound}-\ref{thm:int}. The details are in Appendix \ref{proof:rfnorobust}. 

We note that the results of Theorems \ref{thm:noiselowerbound}-\ref{thm:int} do not require the assumption $\E_{\rho\sim \mathcal N (0, 1)} [\phi' (\rho)] \neq 0$. In fact, the role of this additional condition is clarified by the statement of Theorem \ref{thm:int}: if $\E_{\rho\sim\mathcal N(0, 1)} [\phi' (\rho)] \neq  0$, then $\norm{\Irf(z)}_F$ is of order $\Theta(k\sqrt{N/d})$; otherwise, $\norm{\Irf(z)}_F$ is drastically smaller. This provides an explanation to the qualitatively different behavior of the sensitivity for $\phi(x)=\tanh(x)$ and $\phi(x)=x^2$ displayed in Figure \ref{fig:rf_old_act}. To conclude, the interaction matrix appears to be the key quantity capturing the impact of the activation function on the robustness of the ERM solution.




\section{Main Result for NTK Regression}\label{sec:ntk}

In this section, we provide our law of robustness for the \emph{NTK} model. In particular, we will show that, for a class of \emph{even} activations, the sensitivity of NTK can be $\bigO{1}$ and, therefore, the \emph{model is robust} for some scaling of the parameters.

We consider the following two-layer neural network
\begin{equation}\label{eq:linNN}
    f_{\textup{NN}}(x, w) = \sum_{i = 1}^k \phi\left(W^1_{i:} x\right) - \sum_{i = 1}^k \phi\left(W^2_{i:} x\right).
\end{equation}
Here, the hidden layer contains $2k$ neurons; $\phi$ is an activation function applied component-wise; 
$W^1, W^2\in \mathbb R^{k \times d}$ denote the first and second half of the weights of the hidden layer, respectively; for $j\in \{1, 2\}$, $W^j_{i:}$ denotes the $i$-th row of $W^j$; the first $k$ weights of the second layer are set to $1$, and the last $k$ weights to $-1$. We indicate with $w$ the vector containing the parameters of this model, i.e., $w = [\text{vec}(W^1), \text{vec}(W^2)]\in \R^p$, with $p = 2kd$. For convenience, we initialize the network so that its output is $0$. This has been shown to be necessary to have a robust model after lazy training \cite{dohmatob2022non, wang2022adversarial}. Specifically, we let $w_0=[\text{vec}(W_0^1), \text{vec}(W_0^2)]$ be the vector of the parameters at initialization, where we take $[W^1_0]_{i,j} \distas{}_{\rm i.i.d.} \mathcal N(0, 1 / d)$ and $W^2_0=W^1_0$. Here, with a slight abuse of notation, we use the subscript 0 to refer to the initialization, and not to indicate matrix rows. This readily implies that $f_{\textup{NN}}(x, w_0) =0$, for all $x$.

Now, the \emph{NTK regression model} takes the form 
\begin{equation}
  f_{\textup{NTK}}(x, \theta) \hspace{-.15em}=\hspace{-.15em} \fntk(x)^\top \hspace{-.15em}\theta, \,\,  \fntk(x)\hspace{-.15em} = \hspace{-.15em}\nabla_w f_{\textup{NN}}(x, w) |_{w = w_0}.
\end{equation}
Here, the vector trainable parameters is $\theta \in \R^p$, again with $p = 2kd$, which is initialized with $\theta_0 = w_0$. This is the same model considered by \cite{dohmatob2022non,montanari2022interpolation}. We remark that $f_{\textup{NTK}}(x, \theta)$ is equivalent to the linearization of $f_{\textup{NN}}(x, w)$ around the initial point $w_0$, see e.g. \cite{JacotEtc2018,bartlett2021deep}. 

An application of the chain rule gives
\begin{equation}\label{eq:ntkfeaturemap}
    \fntk(x) = \left[ x \otimes \phi'(W^1_0 x), - x \otimes \phi'(W^2_0 x) \right] =: \left[\fntk'(x), - \fntk'(x) \right],
\end{equation}
where the last equality follows from the fact that $W^1_0 = W^2_0$, and the definition $\fntk'(x) := x \otimes \phi'(W^1_0 x)$. Thus, since $\theta_0 = w_0 = [\text{vec}(W^1_0), \text{vec}(W^2_0)] = [\text{vec}(W^1_0), \text{vec}(W^1_0)]$, we have
\begin{equation}
    f_{\textup{NTK}}(x, \theta_0) = \fntk(x)^\top \theta_0 = \left[\fntk'(x), - \fntk'(x) \right]^\top [\text{vec}(W^1_0), \text{vec}(W^1_0)] = 0.
\end{equation}
This means that our model's output is 0 at initialization. Then, we can use \eqref{eq:sensforlin} to express the sensitivity of the trained NTK regression model as
\begin{equation}\label{eq:sensNTK}
    \mathcal S_{\textup{NTK}}(z) = \norm{z}_2 \norm{\nabla_z \fntk(z)^\top \fntk^\top \kntk^{-1} Y}_2,
\end{equation}
where 
$\fntk \in \R^{N \times p}$ contains in its $i$-th row the feature map of the $i$-th training sample $\fntk(x_i)$ and $\kntk = \fntk \fntk^\top$. The invertibility of the kernel will again follow from the proof of our main result, Theorem \ref{thm:ntk}. Throughout this section, we make the following assumptions. 


\begin{assumption}[Activation function]\label{ass:activationfuncntk}
    The activation function $\phi$ satisfies the following properties:
    \begin{enumerate}
    \item $\phi$ is a non-linear, even function.
    \item Its first order derivative $\phi'$ is an $L$-Lipschitz function.
    \end{enumerate}
\end{assumption}
We restrict our analysis to an even activation function, as this requirement significantly simplifies the derivations.  The impact of the activation on the robustness will also be discussed at the end of this section.


\begin{assumption}[Minimum over-parameterization]\label{ass:overparamntk}
    \begin{equation}\label{eq:overparamntk}
        N \log^8 N = o(kd).
    \end{equation}
\end{assumption}


Eq. \eqref{eq:overparamntk} provides the weakest possible requirement on the number of parameters of the model capable to guarantee interpolation for generic data points, as it leads to a lower bound on the smallest eigenvalue of $\kntk$ \cite{bombari2022memorization}.

\begin{assumption}[High-dimensional data]\label{ass:dlargentk}
    \begin{equation}\label{eq:dlargentk}
        k = \bigO{d}.
    \end{equation}
\end{assumption}


This requirement is purely technical, 
and we leave as a future work the interesting problem of characterizing the sensitivity of the NTK regression model when $d = o(k)$.

We will still use Assumption \ref{ass:datadist} on the data distribution $P_X$, and the only assumption on the labels is that $\norm{Y}_2=\Theta(\sqrt{N})$, which corresponds to the natural requirement that the output is $\Theta(1)$. At this point, we are ready to 
state our main result for NTK regression. 

\begin{figure}[!t]
  \begin{center}
    \includegraphics[width=0.99\textwidth]{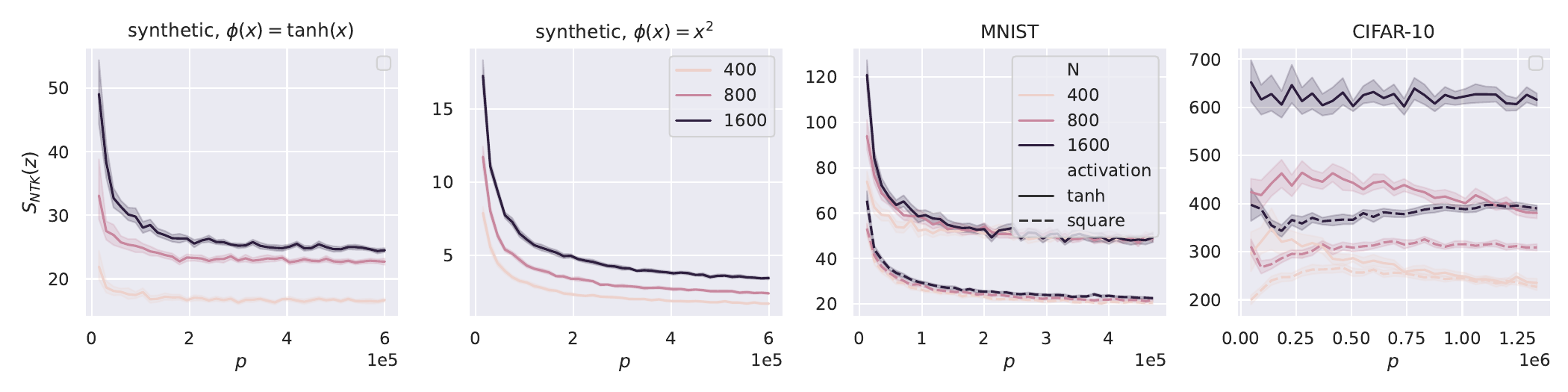}
  \end{center}
 \caption{Sensitivity as a function of the number of parameters $p$, for an NTK model with synthetic data (first and second plot), and with inputs taken from two classes of the MNIST and CIFAR-10 datasets (third and fourth plot respectively). The rest of the setup is similar to that of Figure \ref{fig:rf_old_act}.}
 \label{fig:ntk_old_act}
\end{figure}

\begin{theorem}\label{thm:ntk}
    Let Assumptions \ref{ass:datadist}, \ref{ass:activationfuncntk}, \ref{ass:overparamntk} and \ref{ass:dlargentk} hold. Then, we have
    \begin{equation}\label{eq:NTKs}
           \mathcal S_{\textup{NTK}}(z)  = \bigO{\log k \left( 1 + \sqrt{\frac{N}{k}} \right) \sqrt{\frac{N}{k}}},
    \end{equation}
    with probability at least $1 - N e^{-c \log^2 k} - e^{-c \log^2 N}$ over $X$ and $w_0$.\\
    Furthermore, when $N = \bigO{k}$, \eqref{eq:NTKs} simplifies to 
    \begin{equation}\label{eq:upboundS}
           \mathcal S_{\textup{NTK}}(z) = \bigO{\log k \sqrt{\frac{Nd}{p}}},
    \end{equation}
    where $p = 2dk$ denotes the number of parameters of the model.
\end{theorem}

\noindent
\emph{Proof sketch.} As for the RF model, we crucially separate the contribution of the interaction matrix and of the kernel of the model. Specifically, we upper bound the RHS of \eqref{eq:sensNTK} with
\begin{equation}\label{eq:ubNTK}
 \norm{z}_2   \opnorm{\nabla_z \fntk(z)^\top \fntk^\top} \opnorm{\kntk^{-1}} \norm{Y}_2.
\end{equation}
Now, each term in \eqref{eq:ubNTK} is treated separately. First, by assumption, $\norm{Y}_2 = \Theta(\sqrt{N})$. Next, we have that $\opnorm{\kntk^{-1}} = \lambda_{\min}^{-1}(\kntk) = \bigO{(dk)^{-1}}$, which can be deduced from \cite{bombari2022memorization}, see Lemma \ref{lemma:evminntk}. Finally, the bound on $\opnorm{\nabla_z \fntk(z)^\top \fntk^\top}$ is computed explicitly through an application of the chain-rule (see Lemma \ref{lemma:opnormIntk}), and it critically depends on the data being 0-mean and on the activation being even.
\qed

In words, Theorem \ref{thm:ntk} shows that, for even activations and $k$ scaling super-linearly in $N$, $\mathcal S_{\textup{NTK}}(z) = \bigO{1}$, namely, the NTK model is robust against adversarial perturbations. A few remarks are now in order.




\vspace{-0.5em}

\paragraph{Saturating the lower bound from the universal law of robustness.} We highlight that our upper bound on the sensitivity in \eqref{eq:upboundS} exhibits the same scaling as the lower bound by \cite{bubeck2021a}, thus saturating the universal law of robustness. A lower bound on the sensitivity of the same order is also provided by \cite{dohmatob}, albeit restricted to the setting in which $k=\Theta(d)$ and $k=\bigO{N}$. Note that Theorem \ref{thm:ntk} upper bounds the sensitivity of the model obtained by running gradient descent (over $\theta$) on $f_{\textup{NTK}}(x, \theta)$. It is well known, see e.g. \cite{chizat2019lazy}, that a similar solution is obtained by running gradient descent (over $w$) on $f_{\textup{NN}}(x, w)$. This result holds after suitably rescaling the initialization of $f_{\textup{NN}}(x, w)$, and the rescaling does not affect our Theorem \ref{thm:ntk}. As a result, Theorem \ref{thm:ntk} proves that a class of two-layer neural networks has sensitivity of order $\sqrt{Nd/p}$ (up to logarithmic factors), thus resolving in the affirmative Conjecture 2 of \cite{bubeck2021law}.


\vspace{-0.5em}

\paragraph{Impact of the activation function and numerical results.} Theorem \ref{thm:ntk} applies to even activation functions. At the technical level, the symmetry in $\phi$ directly ``centers'' the kernel in the expression \eqref{eq:sensNTK} of the sensitivity, which largely simplifies our argument. We conjecture that this symmetry may in fact be fundamental to guarantee the robustness of the model. In fact, the first plot of Figure \ref{fig:ntk_old_act} shows that, for $\phi(x)=x^2$, the sensitivity keeps decreasing, as the number of neurons $k$ (and, therefore, the number of parameters $p$) grows. In contrast, for $\phi(x)=\tanh(x)$, the sensitivity plateaus at a value which is an order of magnitude larger than for $\phi(x)=x^2$ (see the second plot of Figure \ref{fig:ntk_old_act}).
This difference is still visible even for real-world datasets (MNIST, CIFAR-10), as displayed in the third and fourth plot of Figure \ref{fig:ntk_old_act}.

\vspace{-0.5em}


\vspace{-0.5em}

\section{Conclusions}

\vspace{-0.5em}


Our paper provides a precise and quantitative characterization of how the robustness of the solution obtained via empirical risk minimization depends on the model at hand: for random features and non-even activations, over-parameterization does not help, and we provide bounds tighter  than the universal law of robustness proposed by \cite{bubeck2021a}; for NTK regression and even activations, the model is robust as soon as $p > dN$, i.e., the universal law of robustness is saturated. Numerical results on synthetic and standard datasets confirm the impact of the model on the robustness of the solution. The present contribution focuses on empirical risk minimization, which represents the cornerstone of deep learning training algorithms. 
In contrast, a number of recent papers has focused on adversarial training, considering the trade-off between generalization and robust error, see e.g. \cite{raghunathan2019adversarial,zhang2019theoretically,dobriban2020provable,carmon2019unlabeled,min2021curious,hamedadversarial} and references therein. We conclude by mentioning that the technical tools developed in this work could be applied also to the models deriving from adversarial training, and we leave such an analysis as future work.

\vspace{-0.5em}




\section*{Acknowledgements}
\vspace{-0.5em}

Simone Bombari and Marco Mondelli were partially supported by the 2019 Lopez-Loreta prize, and the authors would like to thank Hamed Hassani for helpful discussions.

{
\small

\bibliographystyle{plain}
\bibliography{bibliography.bib}

}

\newpage

\appendix

\section{Additional Notations and Remarks}\label{app:notation}
Given a sub-exponential random variable $X$, let $\|X\|_{\psi_1} = \inf \{ t>0 \,\,: \,\,\mathbb E[\exp(|X|/t)] \le 2 \}$.
Similarly, for a sub-Gaussian random variable, let $\|X\|_{\psi_2} = \inf \{ t>0 \,\,: \,\,\mathbb E[\exp(X^2/t^2)] \le 2 \}$.
We use the analogous definitions for vectors. In particular, let $X \in \mathbb R^n$ be a random vector, then $\subGnorm{X} := \sup_{\norm{u}_2=1} \subGnorm{u^\top X}$ and $\subEnorm{X} := \sup_{\norm{u}_2=1} \subEnorm{u^\top X}$. Notice that if a vector has independent, mean 0, sub-Gaussian (sub-exponential) entries, then its sub-Gaussian (sub-exponential). This is a direct consequence of Hoeffding's inequality and Bernstein's inequality (see Theorems 2.6.3 and 2.8.2 in \cite{vershynin2018high}).
It is useful to also define the generic Orlicz norm of order $\alpha$ of a real random variable $X$ as 
\begin{equation}\label{eq:orlicz}
    \norm{X}_{\psi_\alpha} := \inf \{ t>0 \,\,: \,\,\mathbb E[\exp(|X|^\alpha/t^\alpha)] \le 2 \}.
\end{equation}
From this definition, it follows that $\norm{|X|^\gamma}_{\psi_{\alpha / \gamma}} = \norm{X}_{\psi_{\alpha}}^\gamma$. 
Finally, we recall the property that if $X$ and $Y$ are scalar random variables, we have $\subEnorm{XY} \leq \subGnorm{X} \subGnorm{Y}$.

We say that a random variable respects the Lipschitz concentration property if there exists an absolute constant $c > 0$ such that, for every Lipschitz continuous function $\varphi: \RR^d \to \RR$, we have $\E |\varphi(X)| < + \infty$, and for all $t>0$,
\begin{equation}\label{eq:deflipconc}
    \PP\left(\abs{\varphi(x)- \E_X [\varphi(x)]}>t\right) \leq 2e^{-ct^2 / \norm{\varphi}_{\Lip}^2}.
\end{equation}

When we state that a random variable or vector $X$ is sub-Gaussian (or sub-exponential), we implicitly mean $\subGnorm{X} = \bigO{1}$, i.e. it doesn't increase with the scalings of the problem. Notice that, if $X$ is Lipschitz concentrated, then $X - \E[X]$ is sub-Gaussian.
If $X \in \R$ is sub-Gaussian and $\varphi: \R \to \R$ is Lipschitz, we have that $\varphi(X)$ is sub-Gaussian as well. Also, if a random variable is sub-Gaussian or sub-exponential, its $p$-th momentum is upper bounded by a constant (that might depend on $p$). 

In general, we indicate with $C$ and $c$ absolute, strictly positive, numerical constants, that do not depend on the scalings of the problem, i.e. input dimension, number of neurons, or number of training samples. Their value may change from line to line.

Given a matrix $A$, we indicate with $A_{i:}$ its $i$-th row, and with $A_{:j}$ its $j$-th column. Given two matrices $A, B\in \R^{m\times n}$, we denote by $A \circ B$ their Hadamard product, and by $A \ast B=[(A_{1:}\otimes B_{1:}),\ldots,(A_{m:}\otimes B_{m:})]^\top \in\RR^{m\times n^2}$ their row-wise Kronecker product (also known as Khatri-Rao product). We denote $A^{*2} = A \ast A$. We say that a matrix $A\in \R^{n \times n}$ is positive semi definite (p.s.d.) if it's symmetric and for every vector $v \in \R^n$ we have $v^\top A v \geq 0$.

\section{Useful Lemmas}

\begin{lemma}\label{lemma:hammer}
        Let $Z$ be a matrix with rows $z_i \in \mathbb R^{n}$, for $i\in [N]$, such that $N = o(n / \log^4 n)$. Assume that $\{z_i\}_{i\in [N]}$ are independent sub-exponential random vectors with $\psi = \max_i \subEnorm{z_i}$. Let $\eta_{\textup{min}} = \min_i \norm{z_i}_2$ and $\eta_{\textup{max}} = \max_i \norm{z_i}_2$. Set $\xi = \psi K + K'$, and
	$\Delta := C \xi^2 N^{1/4} n^{3/4}$. Then, we have that
	\begin{equation}\label{eq:hm}
		\evmin{ZZ^\top } \geq \eta_{\textup{min}}^2 - \Delta, \qquad \evmax{ZZ^\top } \leq \eta_{\textup{max}}^2 + \Delta
	\end{equation}
	holds with probability at least
	\begin{equation}
	1 -  \exp \left( -cK \sqrt{N} \log \left( \frac{2 n}{N} \right) \right) - \mathbb{P}\left(\eta_{\text{max}} \geq K' \sqrt{n}\right),
	\end{equation}
	where $c$ and $C$ are numerical constants.
\end{lemma}

\begin{proof}
    Following the notation in \cite{hammer}, we define
    \begin{equation}
        B := \sup_{u\in \mathbb R^N : \norm{u}_2=1} \left| \norm{\sum_{i=1}^N u_i z_i}_2^2 - \sum_{i=1}^N u_i^2 \norm{z_i}_2^2 \right|^\frac{1}{2}.
    \end{equation}
    Then, for any $u\in \mathbb R^N$ with unit norm, we have that
    \begin{equation}
    \begin{split}
        \norm{Z u}_2^2 &= \norm{\sum_{i=1}^N u_i z_i}_2^2-\sum_{i=1}^N u_i^2 \norm{z_i}_2^2 + \sum_{i=1}^N u_i^2 \norm{z_i}_2^2  \ge \min_i\norm{z_i}_2^2 - B^2,
    \end{split}
    \end{equation}
    which implies that
    \begin{equation}\label{eq:Beig}
        \evmin{ZZ^\top}= \inf_{u\in \mathbb R^N : \norm{u}_2=1}\norm{Z u}_2^2\geq \min_{i} \norm{z_i}_2^2 - B^2.
    \end{equation}
    In the same way, it can also be proven that
    \begin{equation}\label{eq:Beig2}
        \evmax{ZZ^\top} \leq \max_{i} \norm{z_i}_2^2 + B^2.
    \end{equation}
    In our case, $z_i \in \R ^n$. In the statement of Theorem 3.2 of \cite{hammer}, let's fix $r=1$, $m=N$, and $\theta = (N / n)^{1/4} < 1/4$. Then, we have that the condition required to apply Theorem 3.2 is satisfied, \ie,
    \begin{equation}
        N \log^2 \left(2 \sqrt[4]{\frac{n}{N}}\right) \leq \sqrt{N n}.
    \end{equation}
    By combining \eqref{eq:Beig} and \eqref{eq:Beig2} with the upper bound on $B$ given by Theorem 3.2 of \cite{hammer}, the desired result readily follows.
\end{proof}

\begin{lemma}\label{lemma:opnormlog}
Let $\{x_i\}_{i =1}^N$ be random variables in $\R$ such that they all respect
\begin{equation}
    \P(|x_i| > t) < 2 e^{-ct^\gamma}.
\end{equation}
Then, we have
\begin{equation}
    \max_{i} |x_i| \leq \log^{2/\gamma}{N},
\end{equation}
with probability at least $1 - 2 \exp (- c \log ^2 N)$, where $c$ is a numerical constant.
\end{lemma}

\begin{proof}
For any $i\in [N]$,
\begin{equation}
	\P(|x_i| > \log^{2/\gamma} N) < 2 \exp (-c \log^2 N),
\end{equation}
which gives
\begin{equation}
    \begin{aligned}
        \P(\max_{i} |x_{i}| > \log^{2/\gamma} N) < N \P(|x_1| > \log^{2/\gamma} N) < 2 \exp (\log N -c \log^2 N) <  2 \exp (- c_1 \log^2 N),
    \end{aligned}
\end{equation}
where the second step is obtained through a union bound. This gives the desired result.
\end{proof}

\begin{lemma}\label{lemma:HW}
    Let $z \sim P_Z$ be a mean-0 and Lipschitz concentrated random vector (see \eqref{eq:deflipconc}). 
    Consider
	\begin{equation}
		\Gamma(z) = z \otimes z - \E_z \left[ z \otimes z \right].
	\end{equation}
	Then,
	\begin{equation}
		\norm{\Gamma(z)}_{\psi_1} < C.
	\end{equation}
	where $C$ is a numerical constant.
\end{lemma}

\begin{proof}
	We have that
        \begin{equation}
		\norm{\Gamma(z)}_{\psi_1} = \sup_{\norm{u}_2 = 1} \norm{u^\top \Gamma(z)}_{\psi_1} = \sup_{\norm{U}_F = 1} \subEnorm{z^\top U z - \E_x \left[ z^\top U z \right]}.
	\end{equation}
	Since $z$ is mean-0 and Lipschitz concentrated,  
	we can apply the version of the Hanson-Wright inequality given by Theorem 2.3 in \cite{HWconvex}:
	\begin{equation}
		\P \left(|z^\top U z - \E_x \left[ z^\top U z \right]| > t \right) < 2 \exp \left( -\frac{1}{C_1} \min \left( \frac{t^2}{\norm{U}^2_F}, \frac{t}{\opnorm{U}} \right) \right),
	\end{equation}
	where $C_1$ is a numerical constant. Thus, by Lemma 5.5 of \cite{theoreticalinsghts}, and using $\norm{U}_F \geq \opnorm{U}$, we conclude that
	\begin{equation}
		\norm{\Gamma(z)}_{\psi_1} < C \norm{U}_F = C,
	\end{equation}
	for some numerical constant $C$, which gives the desired result.
\end{proof}

\begin{lemma}\label{lemma:opnormcovariance}
Let $z\in \R^{d}$ be a sub-Gaussian vector such that $\norm{z}_{\psi_2} = K$. Set $\Sigma = \E \left[ zz^\top \right]$.  Then,
\begin{equation}
    \opnorm{\Sigma} \leq C K^2,
\end{equation}
where $C$ is a numerical constant.
\end{lemma}

\begin{proof}
	Let $w$ be the unitary eigenvector associated to the maximum eigenvalue of $\Sigma$. Then,
	\begin{equation}\label{eq:int1}
		\opnorm{\Sigma} = w^\top \Sigma w = \E \left[ w^\top zz^\top w \right]=\E \left[ (w^\top z)^2 \right].
	\end{equation}
	Furthermore, we have that
	\begin{equation}\label{eq:int2}
		\subGnorm{z} := \sup_{w' \text{s.t.} \norm{w'}_2 = 1} \subGnorm{(w')^\top z} \geq \subGnorm{ w^\top z} \geq \frac{1}{C} \sqrt{ \E \left[ (w^\top z)^2 \right]},
	\end{equation}
	where $C$ is a numerical constant, and the last inequality comes from Eq. (2.15) of \cite{vershynin2018high}. By combining \eqref{eq:int1} and \eqref{eq:int2} we get our desired result.
\end{proof}

\begin{lemma}\label{lemma:matiidrows}
    Let $Z$ be an $N \times n$ matrix whose rows $z_i$ are i.i.d. mean-0 sub-Gaussian random vectors in $\R^n$. Let $K = \subGnorm{z_i}$ the sub-Gaussian norm of each row. Then, we have
    \begin{equation}
        \opnorm{ZZ^\top} = K^2 \bigO{N + n},
    \end{equation}
    with probability at least $1 - 2 \exp(-c n)$, for some numerical constant $c$.
\end{lemma}

\begin{proof}  
The result is equivalent to Lemma B.7 of \cite{bombari2022memorization}.
\end{proof}

\begin{lemma}\label{lemma:alphatails}
    Let $\{Y_i\}_{i = 1}^n$ be independent real random variables such that $\norm{Y_i}_{\psi_\alpha} \leq K$ for every $i$ (see \eqref{eq:orlicz}), with $\alpha \in (0, 2]$. Then, for every $t>0$ we have,
    \begin{equation}
    \P \left(\sum_{i=1}^n \left( Y_i - \E[Y_i] \right) > t \right) < 2 \exp \left( - c \left( \frac{t}{K \sqrt n \log^{1/\alpha} n}\right)^\alpha \right),
    \end{equation}
    where $c$ is an absolute constant.
\end{lemma}

\begin{proof}
    Let $Y$ be the vector containing the $Y_i$'s in its entries. Then, an application of Cauchy–Schwarz inequality gives that the function
    \begin{equation}
        f(Y) := \frac{1}{\sqrt{n}} \sum_{i = 1}^n Y_i
    \end{equation}
    is 1-Lipschitz.

    By Lemma 5.6 of \cite{alphatails}, we also have that
    \begin{equation}
        \norm{\max_i|Y_i|}_{\psi_\alpha} \leq C K \log^{1/\alpha} n,
    \end{equation}
    where $C$ is a numerical constant (that might depend on $\alpha$).
    Thus, as $f$ is also convex, we can apply Proposition 3.1 of \cite{alphatails}, obtaining
    \begin{equation}
    \P \left(\frac{1}{\sqrt{n}} \sum_{i=1}^n \left( Y_i - \E[Y_i] \right) > t \right) < 2 \exp \left( - c \frac{t^\alpha}{K^\alpha \log n}\right).
    \end{equation}

    Rescaling $t$ gives the thesis.
\end{proof}


\section{Proofs for Random Features}

In this section, we will indicate with $X \in \R^{N \times d}$ the data matrix, such that its rows are sampled independently from $P_X$ (see Assumption \ref{ass:datadist}). We will indicate with $V \in \R^{k \times d}$ the random features matrix, such that $V_{ij} \distas{}_{\rm i.i.d.}\mathcal{N}(0, 1/d)$. The feature vectors will be therefore indicated by $\phi(XV^\top)$, where $\phi$ is the activation function, applied component-wise to the pre-activations $XV^\top$. We will also use the shorthands $\Phi := \phi(XV^\top) \in \R^{N \times k}$ and $K := \Phi \Phi^\top \in \R^{N \times N}$. We also define $\tilde \Phi := \Phi - \E_X[\Phi]$ and $\tilde K := \tilde \Phi \tilde \Phi^\top$. Notice that, for compactness, we dropped the subscripts \enquote{RF} from these quantities, as this section will only treat the proofs related to Section \ref{sec:rf}. Again for the sake of compactness, we will not re-introduce such quantities in the statements or the proofs of the following lemmas.


\subsection{Proof of Theorem \ref{thm:noiselowerbound}}\label{proof:noiselowerbound}
\begin{proof}[Proof of Theorem \ref{thm:noiselowerbound}]
We have
\begin{equation}\label{eq:lemmamixterm}
    \norm{AY}_2^2 = \norm{A \left( g(X) + \epsilon \right)}_2^2 = \epsilon^\top A^\top A \epsilon +  g(X)^\top A^\top A g(X) + 2 g(X)^\top A^\top A \epsilon.
\end{equation}
As $\epsilon$ is sub-Gaussian, we have $\subGnorm{A \epsilon} \leq C \opnorm{A}$, where $C$ is a numerical constant. Indicating with $M:= \norm{A g(X)}_2$, we can write
\begin{equation}
    \P(|g(X)^\top A^\top A \epsilon| > t) < 2 \exp \left( -c \frac{t^2}{M^2 \opnorm{A}^2} \right).
\end{equation}

Setting $t = y \opnorm{A} M$ in the previous equation, we readily get
\begin{equation}
    \P(|g(X)^\top A^\top A \epsilon| < y \opnorm{A} M) > 1 - 2 \exp \left( -c y^2 \right).
\end{equation}
We will condition on this event for the rest of the proof.

On the other hand, a direct application of Theorem 6.3.2 in \cite{vershynin2018high}, gives
\begin{equation}
    \subGnorm{\norm{A \epsilon}_2 - \varepsilon \norm{A}_F } \leq C_1 \opnorm{A},
\end{equation}
where $C_1$ is an absolute constant (that depends on $\varepsilon$). This means that
\begin{equation}
    \P \left( \left| \norm{A\epsilon}_2 - \varepsilon \norm{A}_F \right| > z \opnorm{A} \right) <  2 \exp \left( -c_1 z^ 2\right).
\end{equation}
This also implies
\begin{equation}
    \P \left( \norm{A\epsilon}_2 \geq \varepsilon \norm{A}_F / \sqrt 2 \right) > 1 -  2 \exp \left( -c_2 \frac{\norm{A}_F^2}{\opnorm{A}^2}\right).
\end{equation}
We will condition also on this other event for the rest of the proof.

Using the triangle inequality, \eqref{eq:lemmamixterm} gives
\begin{equation}
    \norm{AY}_2^2 \geq \epsilon^\top A^\top A \epsilon + M^2 - 2 y \opnorm{A} M \geq \epsilon^\top A^\top A \epsilon - y^2 \opnorm{A}^2 \geq \frac{\varepsilon^2}{2} \norm{A}_F^2 - y^2 \opnorm{A}^2,
\end{equation}
where the second inequality is obtained through minimization over $M$. Setting $y = \norm{A}_F / (2\opnorm{A})$ we get
\begin{equation}
    \norm{AY}_2 \geq  \frac{\varepsilon}{2} \norm{A}_F,
\end{equation}
with probability at least $1 -  \exp \left( -c_3 \norm{A}_F^2 / \opnorm{A}^2 \right)$, which concludes the proof.
\end{proof}

\subsection{Bounds on the Extremal Eigenvalues of the Kernel}\label{app:spectrum}

\begin{lemma}\label{lemma:subGv1}
    Let $v \in \R^d$ be distributed as $V_{j:}$, for some $j\in [k]$, independent from $X$. Then, we have
    \begin{equation}
        \subGnorm{\phi(Xv)} = \bigO{\sqrt{N}},
    \end{equation}
    where the sub-Gaussian norm is intended over the probability space of $v$. This result holds with probability $1$ over $X$.
\end{lemma}

\begin{proof}
    By triangular inequality, we have
    \begin{equation}\label{eq:subGv1}
        \subGnorm{\phi(Xv)} \leq \subGnorm{\phi(Xv) - \E_v \left[ \phi(Xv) \right]} + \subGnorm{\E_v \left[ \phi(Xv) \right]}.
    \end{equation}
    
    Let's bound the two terms separately. As $\phi$ is an $L$-Lipschitz function, and $\sqrt d v$ is a Lipschitz concentrated random vector, we have
    \begin{equation}\label{eq:subGv11}
        \subGnorm{\phi(Xv) - \E_v \left[ \phi(Xv) \right]} \leq C L \frac{\opnorm{X}}{\sqrt{d}} \leq C L \frac{\norm{X}_F}{\sqrt{d}} = CL \sqrt{N},
    \end{equation}
    where in the last equality we used the fact that $\norm{x_i}_2 = \sqrt{d}$.

    For the second term of \eqref{eq:subGv1}, we have
    \begin{equation}
        \subGnorm{\E_v \left[ \phi(Xv) \right]} \leq C \norm{\E_v \left[ \phi(Xv) \right]}_2.
    \end{equation}
    Let $(\phi(Xv))_i$ denote the $i$-th component of the vector $\phi(Xv)$. As $v$ is a Gaussian vector with covariance $I/d$  and since $\norm{x_i} = \sqrt{d}$ for all $i\in [N]$, we have $\E_v \left[ (\phi(Xv))_i \right] = \E_\rho [\phi (\rho)] = \bigO{1}$, where $\rho$ is a standard Gaussian random variable, and the last equality is true since $\phi$ is Lipschitz. Thus, we have
    \begin{equation}\label{eq:subGv12}
        \subGnorm{\E_v \left[ \phi(Xv) \right]} = \bigO{\sqrt N}.
    \end{equation}

    Putting together \eqref{eq:subGv1}, \eqref{eq:subGv11}, and \eqref{eq:subGv12}, we get the thesis.
\end{proof}

\begin{lemma}\label{lemma:subGv2}
    Let $v \in \R^d$ be distributed as $V_{j:}$, for some $j\in [k]$, independent from $X$. Then, we have
    \begin{equation}
        \subGnorm{\phi(Xv) - \E_X \left[ \phi(Xv) \right]} = \bigO{\sqrt N},
    \end{equation}
where the sub-Gaussian norm is intended over the probability space of $v$. This result holds with probability $1$ over $X$.
\end{lemma}

\begin{proof}
    By triangle inequality, we have
    \begin{equation}\label{eq:subGv1new}
        \subGnorm{\phi(Xv) - \E_X \left[ \phi(Xv) \right]} \leq \subGnorm{\phi(Xv)} + \subGnorm{\E_X \left[ \phi(Xv) \right]}.
    \end{equation}
    The first term is $\bigO{\sqrt{N}}$, due to Lemma \ref{lemma:subGv1}.
    
    For the second term, by Jensen inequality, we have
    \begin{equation}
        \subGnorm{\E_X \left[ \phi(Xv) \right]} \leq \E_X \left[\subGnorm{\phi(Xv)} \right] \leq C \E_X \left[\sqrt N \right] = \bigO{\sqrt{N}}.
    \end{equation}

    Thus, the thesis readily follows.
\end{proof}

\begin{lemma}\label{lemma:nequal2}
    Let  $Y := \left( X^{*2} + (\alpha -1) \E_X \left[ X^{*2} \right] \right)$ for some $\alpha \in \R$, where $*$ refers to the Khatri-Rao product, defined in Section \ref{sec:prel}. Then, we have
    \begin{equation}
        \evmin{YY^\top} = \Omega(d^2),
    \end{equation}
    with probability at least $1 - \exp (-c \log^2 N)$ over $X$.
\end{lemma}
\begin{proof}
    First, we want to compare $YY^\top$ with $\tilde Y \tilde Y^\top$, where
    \begin{equation}
        \tilde Y := X^{*2} - \E_X \left[ X^{*2} \right].
    \end{equation}
    Let's also define $\bar Y:= \E_X \left[ X^{*2} \right] := \mathbf{1} \mu^\top$, where  $\mathbf{1} \in \R^N$ denotes a vector full of ones, and $\mu := \E \left[ x_1 \otimes x_1 \right]$. We will refer to the $i$-th row of $X$ with $x_i$. It is straightforward to verify that $Y = \tilde Y + \alpha \bar Y$. Thus, we can write
    \begin{equation}\label{eq:centeringY}
    \begin{aligned}
        YY^\top &= \tilde Y \tilde Y^\top + \alpha^2 \bar Y \bar Y^\top + \alpha \tilde Y \bar Y^\top + \alpha \bar Y \tilde Y^\top \\
        &= \tilde Y \tilde Y^\top + \alpha^2 \norm{\mu}_2^2 \mathbf{1}\mathbf{1}^\top + \alpha \Lambda \mathbf{1}\mathbf{1}^\top + \alpha \mathbf{1}\mathbf{1}^\top \Lambda \\
        &= \tilde Y \tilde Y^\top  + \left( \alpha \norm{\mu}_2 \mathbf{1} + \frac{\Lambda \mathbf{1}}{\norm{\mu}_2}\right) \left( \alpha \norm{\mu}_2 \mathbf{1} + \frac{\Lambda \mathbf{1}}{\norm{\mu}_2}\right)^\top - \frac{\Lambda 
\mathbf{1} \mathbf{1}^\top \Lambda}{\norm{\mu}_2^2} \\
	&\succeq \tilde Y \tilde Y^\top - \frac{\Lambda \mathbf{1} \mathbf{1}^\top \Lambda}{\norm{\mu}_2^2},
    \end{aligned}
    \end{equation}
    where we define $\Lambda$ as a diagonal matrix such that $\Lambda_{ii} = \mu^\top\tilde Y_{i:}$. 

    Since $\tilde Y_{i:} = x_i \otimes x_i - \E[x_i \otimes x_i] = x_i \otimes x_i - \mu$, where $x_i$ is mean-0 and Lipschitz concentrated, by Lemma \ref{lemma:HW}, we have that $\subEnorm{\tilde Y_{i:}} \leq C$, for some numerical constant $C$.
    Furthermore, we have that
    \begin{equation}\label{eq:Ynormcent}
        \left| \norm{\tilde Y_{i:}}_2 - d \right| = \left| \norm{\tilde Y_{i:}}_2 - \norm{x_i \otimes x_i}_2 \right| \leq \norm{\tilde Y_{i:} - x_i \otimes x_i}_2  = \norm{\mu}_2.
    \end{equation}
    where the second step follows from triangle inequality. We can upper-bound the RHS with the following argument:
    \begin{equation}
        \norm{\mu}_2 = \norm{\E_{x_1} \left[ x_1 x_1^\top \right]}_F \leq \sqrt{d} \opnorm{\E_{x_1} \left[ x_1 x_1^\top \right]} = \bigO{\sqrt{d}},
    \end{equation}
    where the last equality is justified by Lemma \ref{lemma:opnormcovariance}, since $x_i$ is sub-Gaussian.
    Then, plugging this last relation in \eqref{eq:Ynormcent}, we have that, for all $i\in [N]$,
    \begin{equation}
        \norm{\tilde Y_{i:}}_2 = \Theta(d) < C_y d,
    \end{equation}
    for some numerical constant $C_y$.
    
    Note that $\tilde Y$ has $N$ independent rows in $\R^{d^2}$ such that $\subEnorm{\tilde Y_{i:}} \leq C$ for all $i$ and $\min_i \norm{\tilde Y_{i:}}_2 = \Omega(d)$. Assumption \ref{ass:dlarge} directly implies $N = o(d^2 / \log^4 d^2)$. Thus, we can apply Lemma \ref{lemma:hammer}, setting $K = 1$ and $K' = C_y$ in the statement, and we get
    \begin{equation}\label{eq:evmintildeY}
        \evmin{\tilde Y \tilde Y^\top} = \Omega(d^2) - cN^{1/4} d^{3/2} = \Omega(d^2),
    \end{equation}
    with probability at least $1 - \exp(-c \sqrt{N})$ over $X$, where $c$ is a numerical constant.

    To conclude the proof, we have to control the last term of \eqref{eq:centeringY}. As $\Lambda_{ii} = \mu^\top\tilde Y_{i:}$ and $\subEnorm{\tilde Y_{i:}} \leq C$, we can write 
    \begin{equation}
    	\P(|\Lambda_{ii}|/\norm{\mu}_2 > t ) < 2 \exp (- c_1 t),
    \end{equation}
    where $c_1$ is a numerical constant, and the probability is intended over $X$. After applying Lemma \ref{lemma:opnormlog} with $\gamma = 1$, the last relation implies that
    \begin{equation}\label{eq:opnormL}
        \opnorm{\Lambda/\norm{\mu}_2} = \bigO{\log^2 N},
    \end{equation}
    with probability at least $1 - 2 \exp (-c_2 \log^2 N)$ over $X$, where $c_2$ is a numerical constant.

    Plugging \eqref{eq:evmintildeY} and \eqref{eq:opnormL} in \eqref{eq:centeringY}, we get that, with probability at least $1 - \exp (-c_3 \log^2 N)$,
    \begin{equation}
        \evmin{Y Y^\top} \geq \evmin{\tilde Y \tilde Y^\top} - \opnorm{\frac{\Lambda \mathbf{1} \mathbf{1}^\top \Lambda}{\norm{\mu}_2^2}} = \Omega(d^2) - N \bigO{\log^4 N} = \Omega(d^2).
    \end{equation}
    where the last equality is again consequence of Assumption \ref{ass:dlarge}.
\end{proof}

\begin{lemma}\label{lemma:expectedfeat}
We have that
\begin{equation}
    \evmin{\E_V \left[K \right]} = \Omega(k),
\end{equation}
with probability at least $1 - \exp(-c \log^2 N)$ over $X$, where $c$ is an absolute constant.
\end{lemma}
\begin{proof}
    We have
    \begin{equation}
        \E_V \left[K\right] = \E_V \left[ \sum_{i = 1}^k \phi(XV_{i:}^\top) \phi(XV_{i:}^\top)^\top \right] = k \E_v \left[\phi(Xv)  \phi(Xv)^\top \right] := kM,
    \end{equation}
    where we used the shorthand $v$ to indicate a random variable distributed as $V_{1:}$. We will also indicate with $x_i$ the $i$-th row of $X$.

    Exploiting the Hermite expansion of $\phi$, we can write
    \begin{equation}
        M_{ij} = \E_v \left[\phi(x_i^\top v) \phi(x_j^\top v)\right] = \sum_{n = 0}^{+ \infty} \mu_n^2 \frac{\left(x_i^\top x_j\right)^n}{d^n} = \sum_{n = 0}^{+ \infty} \mu_n^2 \frac{ \left[ \left( X^{*n} \right)  \left( X^{*n} \right)^\top \right]_{ij}}{d^n},
    \end{equation}
    where $\mu_n$ is the $n$-th Hermite coefficient of $\phi$. Notice that the previous expansion was possible since $\norm{x_i} = \sqrt{d}$ for all $i\in [N]$.
    As $\phi$ is non-linear, there exists $m \geq 2$ such that $\mu_m ^2 > 0$. In particular, we have $M \succeq \mu_m^2 M_m$ in a p.s.d. sense, where we define
    \begin{equation}
        M_m := \frac{1}{d^m} \left( X^{*m} \right)  \left( X^{*m} \right).
    \end{equation}
    
    Let's consider the case $m = 2$. Then, since all the rows of $X$ are mean-0, independent and Lipschitz concentrated, such that $\norm{x_i} = \sqrt{d}$, we can apply Lemma \ref{lemma:nequal2} with $\alpha = 1$ to readily get
    \begin{equation}
        \evmin{M_m} = \Omega(1),
    \end{equation}
    with probability at least $1 - \exp(-c \log^2 N)$ over $X$.
    
    Let's now consider the case $m \geq 3$. For $i \neq j$ we can exploit again the fact that $x_i$ is sub-Gaussian and independent from $x_j$, obtaining
    \begin{equation}
        \P \left( \left| x_i^\top x_j \right| > t \sqrt{d} \right) < \exp \left( -c_1 t^2 \right).
    \end{equation}
    Performing a union bound we also get
    \begin{equation}
        \P \left( \max_{i,j}\left| x_i^\top x_j \right| > t \sqrt{d} \right) < N^2 \exp \left( -c_1 t^2 \right).
    \end{equation}

    Then, by Gershgorin circle theorem, we have that
    \begin{equation}
        \evmin{M_m} \geq 1 - \max_i \sum_{j \neq i} \frac{\left| \left(x_i^\top x_j\right)^m \right| }{d^m}\geq 1 - \frac{N}{d^m} \max_{i, j} \left| ( x_i^\top x_j )^m \right| \geq 1 - \frac{Nt^m}{d^{m/2}},
    \end{equation}
where the last inequality holds with probability $1 - N^2 \exp \left( -c_1 t^2 \right)$. Setting $t = \log N$, we get
    \begin{equation}
        \evmin{M_m} \geq 1 - \frac{N \log^m N}{d^{m/2}} =  1 - \frac{N \log^3 N}{d^{3/2}}\frac{\log^{m - 3} N}{d^{m - 3/2}} = 1 - o \left( \frac{N \log^3 N}{d^{3/2}} \right) = 1 - o(1),
    \end{equation}
    where the last inequality is a consequence of Assumption \ref{ass:dlarge}. This result holds with probability at least $1 - N^2 \exp \left( -c_1 \log^2 N \right) \geq 1 - \exp \left( -c_2 \log^2 N \right)$, and the thesis readily follows.
\end{proof}

\begin{lemma}\label{lemma:evminfeatures}
We have that
    \begin{equation}
        \evmin{K} = \Omega(k).
    \end{equation}
    with probability at least $1 - \exp \left(-c \log^2 N \right)$ over $V$ and $X$, where $c$ is an absolute constant.
\end{lemma}
\begin{proof}
First we define a truncated version of $\Phi$, which we call $\bar \Phi$. Over such truncated matrix, we will be able to use known concentrations results, to prove that $\evmin{\bar \Phi \bar \Phi^\top} = \Omega(\evmin{\bar G})$, where we define $\bar G := \E_V \left[\bar \Phi \bar \Phi^\top \right]$. Finally, we will prove closeness between $\bar G$ and $G$, which is analogously defined as $G := \E_V \left[\Phi \Phi^\top \right]$. Let's introduce the shorthand $v_i := V_{i:}$.
To be precise, we define
\begin{equation}
    \bar \Phi_{:j} = \phi(X v_j) \chi \left(\norm{\phi(X v_j)}_2^2 \leq R \right),
\end{equation}
where $\chi$ is the indicator function. In this case, $\chi = 1$ if $\norm{\phi(X v_j)}^2_2 \leq R$, and $\chi = 0$ otherwise. As this is a column-wise truncation, it's easy to verify that
\begin{equation}
    \Phi \Phi^\top \succeq \bar \Phi \bar \Phi^\top.
\end{equation}

Denoting with $v$ a generic weight $v_i$, we can also write
\begin{equation}
    G = k \E_{v} \left[\phi \left( X v \right) \phi \left(X v \right) ^\top \right] ,
\end{equation}
\begin{equation}
    \bar G = k \E_{v} \left[ \phi \left( X v \right) \phi \left(X v \right)^ \top  \chi \left(\norm{\phi\left(X v\right)}^2_2 \leq R \right)\right]  .
\end{equation}

Since all the columns of $\bar \Phi$ have limited norm with probability one, we can use Matrix Chernoff Inequality (see Theorem 1.1 
of \cite{Tropp2011}). In particular, for $\delta \in [0, 1]$, we have,
\begin{equation}\label{eq:chernoff}
    \P \left( \evmin{\bar \Phi \bar \Phi^\top} \leq ( 1 - \delta) \evmin{\bar G} \right) \leq  N \left( \frac{e^{-\delta}}{(1-\delta)^{1-\delta}}\right)^{\evmin{\bar G} / R},
\end{equation}
where we used the fact that $\evmax{\bar \Phi_{:j} \bar \Phi_{:j}^\top} \leq R$ and $\bar \Phi \bar \Phi^\top = \sum_{j} \bar \Phi_{:j} \bar \Phi_{:j}^\top$.
Setting $\delta = 1/2$ in the previous equation gives
\begin{equation}\label{eq:chernoff2}
  \P \left( \evmin{\bar \Phi \bar \Phi^\top} \leq \evmin{\bar G} / 2 \right) \leq \exp \left(-c \evmin{\bar G}/R  + c_1 \log N \right),
\end{equation}
where this probability is over $V$.

Finally, we prove that the matrices $G$ and $\bar G$ are \enquote{close}. In particular, we have
\begin{equation}
\begin{aligned}
    \opnorm{G - \bar G} & = k \opnorm{ \E_v \left[ \phi \left( X v \right) \phi \left(X v \right)^ \top  \chi \left(\norm{\phi\left(X v \right)}^2_2 > R \right)\right] } \\
    &\leq k \E_v \left[ \opnorm{ \phi \left( X v \right) \phi \left(X v \right)^ \top  \chi \left(\norm{\phi\left(X v \right)}^2_2 > R \right) }\right] \\
    &= k \E_v \left[ \norm{\phi \left( X v \right)}_2^2  \chi \left(\norm{\phi\left(X v \right)}^2_2 > R \right) \right] \\
    &= k\int_{s=0}^{\infty} \P_v \left(\norm{ \phi\left(X v\right)}_2  \chi \left(\norm{\phi\left(X v \right)}^2_2 > R \right)>\sqrt{s}\right) \mathrm{d}s\\
    &= k\int_{s=0}^{\infty} \P_v \left(\norm{ \phi\left(X v\right)}_2 > \max(\sqrt R, \sqrt{s}) \right) \mathrm{d}s,
\end{aligned}
\end{equation}
where we applied Jensen inequality in the second line.
Let's define $\psi := \subGnorm{\phi\left(X v\right)}$, where the sub-Gaussian norm is intended over the probability space of $v$. We have
\begin{equation}\label{eq:afterpsi}
    \begin{aligned}
    \opnorm{G - \bar G} &\leq k \int_{s=0}^{\infty} \exp \left(-c_2 \frac{\max (R, s)}{\psi^2} \right) \mathrm ds \\ 
    &\leq k \int_{s=0}^{\infty} \exp \left(-c_2 \frac{R + s}{2 \psi^2} \right) \mathrm ds \\
    &= C k \psi^2 \exp \left(-c_2 \frac{R}{2 \psi^2} \right),
    \end{aligned}
\end{equation}
where the first inequality follows from the definition of sub-Gaussian random variable, and the second one from $\max(R, s) \geq (R + s) / 2$.

By Lemma \ref{lemma:subGv1}, we have $\psi = \bigO{\sqrt N}$. Then, setting $R = k / \log^2 N$, we get
\begin{equation}
    \opnorm{G - \bar G} \leq C k \psi^2 \exp \left(-c \frac{R}{2 \psi^2} \right) \leq C_1 k \exp \left(-c_2 \frac{k}{N \log^2 N} + \log N \right) = o(k),
\end{equation}
where the last equality is a consequence of Assumption \ref{ass:overparam}. Notice that this happens with probability 1 over $v$.

Due to Lemma \ref{lemma:expectedfeat}, we have that $\evmin{G} = \Omega(k)$, with probability at least $1 - \exp \left( -c_3 \log^2 N \right)$ over $X$. Conditioning on such event, we have
\begin{equation}
    \evmin{\bar G} \geq \evmin{G} - \opnorm{G - \bar G} = \Omega(k) - o(k) = \Omega(k).
\end{equation}

To conclude, looking back at \eqref{eq:chernoff2}, we also get $\evmin{\bar \Phi \bar \Phi^\top} = \Omega(k)$ with probability (over $V$) at least
\begin{equation}
    1 - \exp \left(-c \evmin{\bar G}/R  + c_1 \log N \right) \geq 1 - \exp \left(-c_4 \frac{k \log^2 N}{k}  + c_1 \log N \right) \geq 1 - \exp \left(-c_5 \log^2 N \right).
\end{equation}
This concludes the proof.
\end{proof}

\begin{lemma}\label{lemma:evmincentered}
We have that
    \begin{equation}
        \evmin{\E_V \left[ \tilde K \right]} = \Omega(k),
    \end{equation}
    with probability at least $1 - \exp(-c \log^2 N)$ over $X$, where $c$ is an absolute constant.
\end{lemma}

\begin{proof}
    Indicating with $x_i$ the $i$-th row of $X$, and with $\tilde \phi(x_i^\top v) := \phi(x_i^\top v) - \E_X \left[ \phi(x_1^\top v)\right]$, we can write 
    \begin{equation}
        \E_V \left[\tilde \Phi \tilde \Phi ^\top \right] = \E_V \left[ \sum_{i = 1}^k \tilde \phi(XV_{i:}^\top)\tilde \phi(XV_{i:}^\top)^\top \right] = k \E_v \left[\tilde \phi(Xv) \tilde \phi(Xv)^\top \right] := kM,
    \end{equation}
where we use the shorthand $v$ to indicate a random variable distributed as $V_{1:}$.
    
    Exploiting the Hermite expansion of $\phi$, we can write
    \begin{equation}
    \begin{aligned}
        M_{ij} &= \E_v \left[\tilde \phi(x_i^\top v) \tilde \phi(x_j^\top v)\right] \\
        &= \E_v \left[ \left( \phi(x_i^\top v) - \E_z \left[ \phi(z^\top v)\right] \right) \left( \phi(x_j^\top v) - \E_y \left[ \phi(y^\top v)\right]\right) \right]\\
        &= \E_{vzy} \left[ \phi(x_i^\top v) \phi(x_j^\top v) + \phi(z^\top v) \phi(y^\top v) - \phi(x_i^\top v) \phi(y ^\top v)  - \phi(x_j^\top v) \phi(z^\top v) \right] \\
        &= \E_{zy} \left[ \sum_{n = 0}^{+ \infty} \mu_n^2 \frac{\left(x_i^\top x_j\right)^n + \left(z^\top y\right)^n - \left(x_i^\top y\right)^n - \left(x_j^\top z\right)^n}{d^n} \right],
    \end{aligned}
    \end{equation}
    where $\mu_n$ is the $n$-th Hermite coefficient of $\phi$, and $z \sim P_X$ and $y \sim P_X$ are two independent random variables, distributed as the data and independent from them. Notice that the previous expansion was possible since $\norm{x_i} = \sqrt{d}$ for all $i\in [N]$.
    As $\phi$ is non-linear, there exists $n \geq 2$ such that $\mu_n ^2 > 0$. Let $m$ denote one of such $n$'s (i.e, $m$ is such that $\mu_m ^2 > 0$), and define $\phi^*$ as the function that shares the same Hermite expansion as $\phi$, but with its $m$-th coefficient being $0$. We have
    \begin{equation}
    \begin{aligned}
        M_{ij} &= \E_{zy} \left[ \sum_{n = 0}^{+ \infty} \mu_n^2 \frac{\left(x_i^\top x_j\right)^n + \left(z^\top y\right)^n - \left(x_i^\top y\right)^n - \left(x_j^\top z\right)^n}{d^n} \right] \\
        &= \E_{zy} \left[ \sum_{n \neq m} \mu_n^2 \frac{\left(x_i^\top x_j\right)^n + \left(z^\top y\right)^n - \left(x_i^\top y\right)^n - \left(x_j^\top z\right)^n}{d^n} \right] \\
        & \qquad + \frac{\mu_{m}^2}{d^{m}} \left( \left(x_i^\top x_j\right)^m + \E_{zy} \left[ \left(z^\top y\right)^m - \left(x_i^\top y\right)^m - \left(x_j^\top z\right)^m \right] \right) \\
        &= \E_v \left[\tilde \phi^*(x_i^\top v) \tilde \phi^*(x_j^\top v)\right] + \frac{\mu_{m}^2}{d^{m}} \left( \left(x_i^\top x_j\right)^m + \E_{zy} \left[ \left(z^\top y\right)^m - \left(x_i^\top y\right)^m - \left(x_j^\top z\right)^m \right] \right) \\
        &=: \E_v \left[\tilde \phi^*(X v) \tilde \phi^*(X v)^\top\right]_{ij} + \mu_m^2 \left[M_m\right]_{ij}.
    \end{aligned}
    \end{equation}
    From the last equation we deduce that $M \succeq \mu_m^2 M_m$, since $\E_v \left[\tilde \phi^*(X v) \tilde \phi^*(X^\top v)\right]$ is PSD. Proving the thesis becomes therefore equivalent to prove that $\evmin{M_m} = \Omega(1)$.

    Let's consider the case $m = 2$. We can write
    \begin{equation}
    \begin{aligned}
        \left[M_m\right]_{ij} &= \frac{1}{d^{2}} \E_{zy}\left[ \left( x_i^\top x_j - z^\top x_j + x_i^\top y - z^\top y \right)\left( x_i^\top x_j + z^\top x_j - x_i^\top y - z^\top y \right) \right] \\
        &= \frac{1}{d^{2}} \E_{zy}\left[ \left( (x_i - z)^\top (x_j +  y) \right)\left( (x_i + z)^\top (x_j - y)^\top \right) \right],
    \end{aligned}
    \end{equation}
    as the cross-terms will be linear in at least one between $y$ and $z$, and $\E[x] = \E[z] = 0$.

    Going back in matrix notation, we can write
    \begin{equation}
        M_m = \frac{1}{d^{2}} \E_{ZY}\left[ \left( (X + Z) (X - Y)^\top \circ (X - Z) (X + Y)^\top \right)\right],
    \end{equation}
    where $Y$ and $Z$ are $N \times d$ matrices respectively containing $N$ copies of $y$ and $z$ in their rows.

    We can now expand this term as
    \begin{equation}\label{eq:YZexpect}
    \begin{aligned}
        M_m &= \frac{1}{d^2} \E_{ZY}\left[ \left( (X - Z) * (X + Z) \right) \left( (X + Y) * (X - Y) \right)^\top  \right] \\
        &= \frac{1}{d^2}  \E_{Z}\left[ (X - Z) * (X + Z) \right] \E_{Y}\left[  (X + Y) * (X - Y) \right] ^\top.
    \end{aligned}
    \end{equation}
    
    Since $\E[Y] = \E[Z] = 0$, the cross-terms cancel out again, giving
    \begin{equation}
        \E_{Z}\left[ (X - Z) * (X + Z)\right] =  X * X - \E_{Z}\left[ Z * Z \right] = X * X - \E_{X}\left[X * X \right],
    \end{equation}
    where the last equality is because all the rows of $Z$ are distributed accordingly to $P_X$. The same equation holds for the term with $Y$ in \eqref{eq:YZexpect}. Thus, we can conclude that
    \begin{equation}
        \evmin{M_m} = \frac{1}{d^2} \evmin{TT^\top},
    \end{equation}
    where we defined $T = X * X - \E_{X}\left[X * X \right]$. As all the rows of $X$ are mean-0, independent and Lipschitz concentrated, such that $\norm{x_i} = \sqrt{d}$, we can apply Lemma \ref{lemma:nequal2} with $\alpha = 0$ to readily get
    \begin{equation}
        \evmin{M_m} = \Omega(1),
    \end{equation}
    with probability at least $1 - \exp(-c \log^2 N)$ over $X$.

    Let's now consider the case $m \geq 3$. Since $x$ is a sub-Gaussian random variable, we have that $\subGnorm{\frac{x^\top y}{\sqrt{d}}} = C$ in the probability space of $x$. Notice that this holds for all $y$, as $\norm{y}_2 = \sqrt{d}$. Then, its $m$-th momentum its bounded by $m^{m/2}$, up to a multiplicative constant. As $m$ doesn't depend by the scalings of the problem, we can simply write
    \begin{equation}
        \E_x \left[ \left( \frac{x^\top y}{\sqrt{d}} \right)^m \right] = \bigO{1}.
    \end{equation}
    In the same way, we can prove that
    \begin{equation}
        \E_x \left[ \left( \frac{x^\top x_i}{\sqrt{d}} \right)^m  \right] = \bigO{1}, \qquad
        \E_y \left[ \left( \frac{y^\top x_j}{\sqrt{d}} \right)^m  \right] = \bigO{1}.
    \end{equation}
    Also, for $i \neq j$ we can exploit again the fact that $x_i$ is sub-Gaussian and independent from $x_j$, obtaining
    \begin{equation}
        \P \left( \left| x_i^\top x_j \right| > t \sqrt{d} \right) < \exp \left( -c t^2 \right).
    \end{equation}
    Performing a union bound we also get
    \begin{equation}
        \P \left( \max_{i \neq j}\left| x_i^\top x_j \right| > t \sqrt{d} \right) < N^2 \exp \left( -c t^2 \right).
    \end{equation}
    
    By Gershgorin circle theorem, and setting $t = \log N$, we have that
    \begin{equation}
    \begin{aligned}
        \evmin{M_m} &\geq 1 - \bigO{d^{-m/2}} - \max_i \sum_{j \neq i} \frac{\left| \left(x_i^\top x_j\right)^m + \E_{xy} \left[ \left(x^\top y\right)^m - \left(x_i^\top y\right)^m - \left(x_j^\top x\right)^m \right] \right| }{d^m} \\
        &\geq 1 - \bigO{d^{-m/2}}- N \max_{i \neq j}  \frac{\left| \left(x_i^\top x_j\right)^m + \E_{xy} \left[ \left(x^\top y\right)^m - \left(x_i^\top y\right)^m - \left(x_j^\top x\right)^m \right] \right| }{d^m} \\
        &\geq 1 - \bigO{Nd^{-m/2}}- \frac{N \log^m N}{d^{m / 2}} = 1 - o(1),
    \end{aligned}
    \end{equation}
    where the third inequality holds with probability at least $1 - N^2 \exp \left( -c \log^2 N \right) = 1 - \exp \left( -c_1 \log^2 N \right)$, and the last equality is a consequence of Assumption \ref{ass:dlarge}. This concludes the proof.
\end{proof}

\begin{lemma}\label{th:evmincentfeatures}
We have that
    \begin{equation}
        \evmin{\tilde K} = \Omega(k),
    \end{equation}
    with probability at least $1 - \exp(-c \log^2 N)$ over $V$ and $X$, where $c$ is an absolute constant.
\end{lemma}
\begin{proof}
    The proof follows the same exact path as Lemma \ref{lemma:evminfeatures}. The only difference is that now $\psi$ (defined right before \eqref{eq:afterpsi}) will be $\psi := \subGnorm{\phi\left(X v\right) - \E_X\left[\phi\left(X v\right) \right]}$. To successfully conclude the proof, we have to show that $\psi = \bigO{\sqrt{N}}$, which is done in Lemma \ref{lemma:subGv2}.
\end{proof}

\begin{lemma}\label{lemma:evmaxfeatures}
We have that
    \begin{equation}
        \opnorm{\tilde K} = \bigO{\frac{k}{d} (N + d)},
    \end{equation}
    with probability at least $1 - k \exp(-c d)$ over $V$ and $X$.
\end{lemma}
\begin{proof}
We have that
\begin{equation}\label{eq:Ktblocks}
    \tilde K = \tilde \Phi \tilde \Phi^\top = \sum_{i = 1}^{\ceil{k / d}} \tilde\Phi^i  (\tilde\Phi^i)^\top,
\end{equation}
where we define $\tilde{\Phi}^i$ as the $N \times d$ matrix containing the columns of $\tilde\Phi$ between $(i-1) d + 1$ and $id$. If $i = \ceil{k / d}$, we have that the number of columns is $k - d (\ceil{k / d} - 1) \leq d$.

Every $\tilde \Phi^i$ has mean-0, i.i.d. rows $\tilde \phi(V^i x_j)$ (in the probability space of $X$), where $V^i$ contains the rows of $V$ between $(i-1) d + 1$ and $id$. Notice that, besides corner cases, $V^i$ is a $d \times d$ matrix.
Since the $x_j$ are Lipschitz concentrated, we also have
\begin{equation}
    \subGnorm{\tilde \phi(V^i x_j)} \leq C \opnorm{V^i} = \bigO{1},
\end{equation}
where the last equality follows from Theorem 4.4.5 of \cite{vershynin2018high} and from the fact that both dimensions of $V^i$ are less or equal than $d$, and holds with probability $1 - \exp(-cd)$ over $V^i$.
Conditioning on such event, a straightforward application of Lemma \ref{lemma:matiidrows} gives
\begin{equation}
    \opnorm{\tilde \Phi^i (\tilde \Phi^i)^\top} = \bigO{N + d},
\end{equation}
with probability at least $1 - \exp(-c_1 d)$ over $X$.

Thus, performing a union bound over the $i$-s, we get
\begin{equation}
    \max_i \opnorm{\Phi^i (\tilde \Phi^i)^\top} = \bigO{N + d},
\end{equation}
with probability at least $1 - k \exp(-c_2 d)$.
Conditioning over such event, we have
\begin{equation}
    \opnorm{\tilde K} \leq \sum_{i = 1}^{\ceil{k / d}} \opnorm{\tilde \Phi^i  (\tilde \Phi^i)^\top} = \bigO{\frac{k}{d} (N + d)},
\end{equation}
which gives the thesis.
\end{proof}

\subsection{Centering and Estimate of \texorpdfstring{$\norm{A(z)}_F$}{Lg}}\label{app:centering}

In the following section it will be convenient to define $E(z) :=\nabla_z \phi(V z)^\top \phi(X V^\top)^\top$. Notice that this gives $A(z) = E(z) K^{-1}$. Furthermore, we will use the notation $\tilde \cdot$ to indicate quantities involving a centering with respect to the random variable $X$. In particular, we indicate
\begin{equation}
    \tilde \Phi = \Phi - \E_X[\Phi], \qquad \tilde K = \tilde \Phi \tilde \Phi^\top, \qquad \tilde E := \nabla_z \phi(V z)^\top \tilde \Phi^\top, \qquad \tilde A = \tilde E \tilde K^{-1}.
\end{equation}
For the sake of compactness, we will not re-introduce any of these quantities in the statements or the proofs of the following lemmas, and we will drop the dependence on $z$ of $A$, $\tilde A$, $E$ and $\tilde E$.

Note the results of this section will be applied only when $\nu := \E_{x_i} \left[ \phi(V x_i) \right]$ is not the all-0 vector. Therefore, we will assume $\norm{\nu}_2 > 0$ through all the proofs here. Along the section it will be convenient to recall that $\nabla_z \phi(V z)^\top = V^\top \text{diag} \left( \phi' (Vz)\right)$.

\begin{lemma} 
Let $\mathbf{1} \in \R^N$ denote a vector full of ones. Then,
    the kernel matrix $K$ can be written as
    \begin{equation}
        K = \tilde K + \eta \eta^\top - \lambda \lambda^\top,
    \end{equation}
    where
    \begin{equation}\label{eq:eta}
        \eta := \norm{\nu}_2 \mathbf{1} + \tilde \Phi  \frac{\nu}{\norm{\nu}_2} = \Phi \frac{\nu}{\norm{\nu}_2},
    \end{equation}
    \begin{equation}\label{eq:lambda}
        \lambda := \tilde \Phi \frac{\nu}{\norm{\nu}_2}.
    \end{equation}
\end{lemma}

\begin{proof}
Note that $\Phi = \tilde \Phi + \mathbf{1} \nu^\top$. Then,
\begin{equation}
    \begin{aligned}
        K &= \left( \tilde \Phi + \mathbf 1 \nu^\top \right) \left( \tilde \Phi + \mathbf 1 \nu^\top \right)^\top \\
        &= \tilde \Phi \tilde \Phi^\top + \tilde \Phi  \nu\mathbf{1} ^\top + \mathbf{1} \nu^\top \tilde \Phi ^\top + \norm{\nu}_2^2 \mathbf{1} \mathbf{1}^\top \\
        &= \tilde \Phi \tilde \Phi^\top + \left( \norm{\nu}_2 \mathbf{1} + \frac{\tilde \Phi \nu}{\norm{\nu}_2} \right)\left( \norm{\nu}_2 \mathbf{1} + \frac{\tilde \Phi \nu}{\norm{\nu}_2} \right)^\top - \frac{\tilde \Phi \nu \nu^\top \tilde \Phi^\top}{\norm{\nu}_2^2} \\
        &= \tilde \Phi \tilde \Phi^\top + \eta \eta^\top - \lambda \lambda^\top.
    \end{aligned}
\end{equation}
\end{proof}

\begin{lemma}\label{lemma:EtoEtilde}
We have that
\begin{equation}
\left| \norm{E K ^{-1}}_F - \norm{\tilde E K ^{-1}}_F \right| = \bigO{\frac{\sqrt{N + d}}{d}},
\end{equation}
with probability at least $1 - \exp(-c \log^2 N) - k \exp(-c d)$ over $X$ and $V$, where $c$ is a numerical constant.
\end{lemma}
\begin{proof}
    We have
    \begin{equation}
        E = \tilde E + \nabla_z \phi(Vz)^\top \nu \mathbf 1^\top.
    \end{equation}
    By triangle inequality, we can write
    \begin{equation}\label{eq:EEt2fact}
        \left| \norm{E K^{-1}}_F - \norm{\tilde E K^{-1}}_F \right| \leq \norm{\nabla_z \phi(Vz)^\top \nu \mathbf 1^\top K^{-1}}_F = \norm{\nabla_z \phi(Vz)^\top \nu}_2 \norm{K^{-1} \mathbf 1}_2.
    \end{equation}
    Let's look at the two factors separately. The first one can be upper bounded as follows
    \begin{equation}\label{eq:fact1forEEt}
        \norm{\nabla_z \phi(Vz)^\top \nu}_2 = \norm{V^\top \text{diag} \left( \phi' (Vz)\right) \nu}_2 \leq \opnorm{V} L \norm{\nu}_2 = \bigO{\sqrt{\frac{k}{d}} \norm{\nu}_2},
    \end{equation}
    where the first inequality holds as $\phi$ is $L$-Lipschitz, and the second equality holds with probability at least $1 - e^{-cd}$ due to \eqref{eq:opnV}.
        Notice that from \eqref{eq:eta} and \eqref{eq:lambda}, we can write
    \begin{equation}
        \mathbf{1} = \frac{\eta}{\norm{\nu}_2}  - \frac{\lambda}{\norm{\nu}_2}.
    \end{equation}
    By plugging this in the second factor of \eqref{eq:EEt2fact} and applying the triangle inequality, we have
    \begin{equation}\label{eq:twotermsagain}
        \norm{K^{-1} \mathbf 1}_2 = \frac{1}{\norm{\nu}_2} \norm{K^{-1} \left( \eta - \lambda \right) }_2 \leq \frac{1}{\norm{\nu}_2} \left( \norm{K^{-1} \eta}_2 + \norm{K^{-1} \lambda}_2 \right).
    \end{equation}
    
    Again, let's control these last two terms separately. Since we can write $\eta = \Phi \nu / \norm{\nu}_2$, we get
    \begin{equation}\label{eq:term1fact2}
        \norm{K^{-1} \eta}_2 = \norm{K^{-1} \Phi \nu / \norm{\nu}_2}_2 \leq \opnorm{K^{-1} \Phi} = \opnorm{\Phi^+} = \evmin{K}^{-1/2} = \bigO{\sqrt{\frac{1}{k}}},
    \end{equation}
    where the last equality holds with probability at $1 - \exp(-c \log^2 N)$, due to Lemma \ref{lemma:evminfeatures}. We will condition on such event for the rest of the proof.
    For the second term, we have
    \begin{equation}\label{eq:term2fact2}
        \norm{K^{-1} \lambda}_2 = \norm{K^{-1} \tilde \Phi \nu / \norm{\nu}_2}_2 \leq \opnorm{K^{-1}} \opnorm{\tilde \Phi} = \evmin{K}^{-1} \bigO{\sqrt{\frac{k(N + d)}{d}}} = \bigO{\sqrt{\frac{N + d}{kd}}},
    \end{equation}
    where the first equality holds due to Lemma \ref{lemma:evmaxfeatures} with probability at least $1 - k \exp(-c_1 d)$. Let's condition on this event too.
    Therefore, putting together \eqref{eq:twotermsagain}, \eqref{eq:term1fact2} and \eqref{eq:term2fact2}, we have
    \begin{equation}
        \norm{K^{-1} \mathbf 1}_2 = \frac{1}{\norm{\nu}_2} \left(\bigO{\sqrt{\frac{1}{k}}} + \bigO{\sqrt{\frac{N + d}{kd}}}\right) = \frac{1}{\norm{\nu}_2} \bigO{\sqrt{\frac{N + d}{kd}}}.
    \end{equation}
    By combining this last expression with \eqref{eq:EEt2fact}-\eqref{eq:fact1forEEt}, we conclude
    \begin{equation}
        \left| \norm{E K^{-1}}_F - \norm{\tilde E K^{-1}}_F \right|= \bigO{\sqrt{\frac{k}{d}} \norm{\nu}_2} \frac{1}{\norm{\nu}_2} \bigO{\sqrt{\frac{N + d}{kd}}} = \bigO{\frac{\sqrt{N + d}}{d}},
    \end{equation}
    which gives the thesis.
\end{proof}

\begin{lemma}\label{lemma:removelambda}
We have that
    \begin{equation}
        \left| \norm{\tilde E K ^{-1}}_F - \norm{\tilde E \left( K + \lambda \lambda^\top \right) ^{-1}}_F \right|= \bigO{\frac{\sqrt{N + d}}{d}},
    \end{equation}
    with probability at least $1 - \exp(-c \log^2 N) - k \exp(-c d)$ over $X$ and $V$, where $c$ is a numerical constant.
\end{lemma}

\begin{proof}
    A direct application of the Sherman-Morrison formula gives
    \begin{equation}\label{eq:shermorr1}
        \left( K + \lambda \lambda^\top\right)^{-1} = K^{-1} - \frac{ K^{-1} \lambda \lambda^\top K^{-1}}{1 + \lambda ^\top K^{-1} \lambda}.
    \end{equation}
    By triangle inequality we know that $\left| \norm{a}_2 - \norm{b}_2 \right|\leq \norm{a - b}_2$. This implies
    \begin{equation}\label{eq:KKl}
        \left| \norm{\tilde E K^{-1}}_F - \norm{\tilde E \left( K + \lambda \lambda^\top \right) ^{-1}}_F \right| \leq \norm{\tilde E K^{-1} - \tilde E \left( K + \lambda \lambda^\top \right) ^{-1}}_F = \norm{\tilde E \frac{ K^{-1} \lambda \lambda^\top K^{-1}}{1 + \lambda ^\top K^{-1} \lambda}}_F,
    \end{equation}
    where the last equality is a consequence of \eqref{eq:shermorr1}.
    This term can be bounded as
    \begin{equation}\label{eq:KKl2}
        \norm{\tilde E \frac{ K^{-1} \lambda \lambda^\top K^{-1}}{1 + \lambda ^\top K^{-1} \lambda}}_F \leq \opnorm{\tilde E} \norm{\frac{ K^{-1} \lambda \lambda^\top K^{-1}}{1 + \lambda ^\top K^{-1} \lambda}}_F = \opnorm{\tilde E} \frac{\lambda^\top K^{-2} \lambda}{1 + \lambda ^\top K^{-1} \lambda}.
    \end{equation}
    The first factor can be bounded as follows
    \begin{equation}\label{eq:opEt1}
        \opnorm{\tilde E} = \opnorm{\nabla_z \phi(V z)^\top \tilde \Phi^\top} \leq \opnorm{V^\top \text{diag} \left( \phi' (Vz)\right)} \opnorm{\tilde \Phi} = \bigO{\sqrt{\frac{k}{d}}}\bigO{\sqrt{\frac{k(N + d)}{d}}} = \bigO{k \frac{\sqrt{N + d}}{d}},
    \end{equation}
    with probability at least $1 - k \exp(-c d)$. Here, we used the fact that $\phi$ is Lipschitz, and we applied Lemma \ref{lemma:evmaxfeatures} and Equation \eqref{eq:opnV}. Furthermore, since $\opnorm{M} M - M^2$ is PSD for every symmetric PSD matrix $M$, the following relation holds
    \begin{equation}\label{eq:MopM}
        \lambda^\top  K^{-2} \lambda \leq \opnorm{K^{-1}} \lambda^\top K^{-1} \lambda.
    \end{equation}

    Thus, 
    \begin{equation}
        \frac{\lambda^\top K^{-2} \lambda}{1 + \lambda ^\top K^{-1} \lambda} \leq \opnorm{K^{-1}} \frac{\lambda^\top K^{-1} \lambda}{1 + \lambda ^\top K^{-1} \lambda} \leq \opnorm{K^{-1}} = \frac{1}{\evmin{K}} = \bigO{\frac{1}{k}},
    \end{equation}
    where the last equality holds with probability at least $1 - \exp(-c \log^2 N)$, because of Lemma \ref{lemma:evminfeatures}.
    Plugging the last equation in \eqref{eq:KKl2} together with \eqref{eq:opEt1}, and comparing with \eqref{eq:KKl}, gives the desired result.
\end{proof}

\begin{lemma}\label{lemma:KtoKtilde}
We have that
    \begin{equation}
        \left| \norm{\tilde E \left( \tilde K + \eta \eta^\top\right)^{-1}}_F - \norm{\tilde E \tilde K ^{-1}}_F \right| = \bigO{\frac{\sqrt{N + d}}{d}},
    \end{equation}
    with probability at least $1 - \exp(-c \log^2 N) - k \exp(-c d)$ over $X$ and $V$, where $c$ is a numerical constant.
\end{lemma}

\begin{proof}
    An application of the Sherman-Morrison formula gives
    \begin{equation}\label{eq:shermorr}
        \left( \tilde K + \eta \eta^\top\right)^{-1} = \tilde K^{-1} - \frac{\tilde K^{-1} \eta  \eta^\top \tilde K^{-1}}{1 + \eta ^\top \tilde K^{-1} \eta}.
    \end{equation}
    This implies
    \begin{equation}\label{eq:KKt}
        \left| \norm{\tilde E \left( \tilde K + \eta \eta^\top\right)^{-1}}_F - \norm{\tilde E \tilde K ^{-1}}_F \right| \leq \norm{\tilde E \left( \tilde K + \eta \eta^\top\right)^{-1} - \tilde E \tilde K ^{-1}}_F = \norm{\tilde E \frac{\tilde K^{-1} \eta  \eta^\top \tilde K^{-1}}{1 + \eta ^\top \tilde K^{-1} \eta}}_F,
    \end{equation}
    where the last equality is a consequence of \eqref{eq:shermorr}.
    This term can be bounded as
    \begin{equation}\label{eq:KKt2}
        \norm{\tilde E \frac{\tilde K^{-1} \eta  \eta^\top \tilde K^{-1}}{1 + \eta ^\top \tilde K^{-1} \eta}}_F \leq \opnorm{\tilde E} \norm{\frac{\tilde K^{-1} \eta  \eta^\top \tilde K^{-1}}{1 + \eta ^\top \tilde K^{-1} \eta}}_F = \opnorm{\tilde E} \frac{\eta^\top \tilde K^{-2} \eta}{1 + \eta ^\top \tilde K^{-1} \eta}.
    \end{equation}
    The first factor can be bounded as in \eqref{eq:opEt1}, in Lemma \ref{lemma:removelambda}, giving $\opnorm{\tilde E} = \bigO{k \frac{\sqrt{N + d}}{d}}$ with probability at least $1 - k \exp(-c d)$.
    Similarly to \eqref{eq:MopM}, we can write
    \begin{equation}
        \eta^\top \tilde K^{-2} \eta \leq \opnorm{\tilde K^{-1}} \eta^\top \tilde K^{-1} \eta.
    \end{equation}

    Thus, 
    \begin{equation}
        \frac{\eta^\top \tilde K^{-2} \eta}{1 + \eta ^\top \tilde K^{-1} \eta} \leq \opnorm{\tilde K^{-1}} \frac{\eta^\top \tilde K^{-1} \eta}{1 + \eta ^\top \tilde K^{-1} \eta} \leq \opnorm{\tilde K^{-1}} = \frac{1}{\evmin{\tilde K}} = \bigO{\frac{1}{k}},
    \end{equation}
    where the last equality holds with probability at least $1 - \exp(-c \log^2 N)$, because of Lemma \ref{th:evmincentfeatures}.
    Plugging the last equation in \eqref{eq:KKt2} together with the bound on $\opnorm{\tilde E}$, and comparing with \eqref{eq:KKt}, gives the desired result.
\end{proof}

\begin{lemma}\label{lemma:totalcentering}
We have that
    \begin{equation}
        \left| \norm{A}_F - \norm{\tilde A}_F \right| = \bigO{\frac{\sqrt{N + d}}{d}},
    \end{equation}
    with probability at least $1 - \exp(-c \log^2 N) - k \exp(-c d)$ over $X$ and $V$, where $c$ is a numerical constant.
\end{lemma}

\begin{proof}
    Recalling that $A = EK^{-1}$ and $\tilde A = \tilde E \tilde K^{-1}$, by triangle inequality we have
    \begin{equation}
    \begin{aligned}
        \left| \norm{A}_F - \norm{\tilde A}_F \right| & \leq \left| \norm{E K ^{-1}}_F - \norm{\tilde E K ^{-1}}_F \right| \\
        &\qquad + \left| \norm{\tilde E K ^{-1}}_F - \norm{\tilde E \left( K + \lambda \lambda^\top \right) ^{-1}}_F \right| \\
        &\qquad + \left| \norm{\tilde E \left( \tilde K + \eta \eta^\top\right)^{-1}}_F - \norm{\tilde E \tilde K ^{-1}}_F \right| \\
        &= \bigO{\frac{\sqrt{N + d}}{d}}
    \end{aligned}
    \end{equation}
    where the inequality holds because of Lemma \ref{lemma:EtoEtilde}, Lemma \ref{lemma:removelambda}, and Lemma \ref{lemma:KtoKtilde}, each of them with probability at least $1 - \exp(-c \log^2 N) - k \exp(-c d)$ over $X$ and $V$. This gives the thesis.
\end{proof}

At this point, we are ready to prove Theorem \ref{thm:Az}.


\begin{proof}[Proof of Theorem \ref{thm:Az}]
    We have
    \begin{equation}
    \begin{aligned}
        \norm{A(z)}_F &\leq \norm{\nabla_z \Phi(z)^\top \tilde \Phi^\top \tilde K^{-1}}_F + \bigO{\sqrt {N + d}/d} \\
        &\leq \norm{\nabla_z \Phi(z)^\top \tilde \Phi^\top}_F \opnorm{\tilde K^{-1}} + \bigO{\sqrt {N + d}/d} \\
        &= \lambda_{\min}^{-1}(K) \mathcal I(z) + \bigO{\sqrt {N + d}/d},
    \end{aligned}
    \end{equation}
    where the first inequality is justified by Lemma \ref{lemma:totalcentering}
    and holds with probability at least $1 - \exp(-c \log^2 N) - k \exp(-c d)$ over $X$ and $V$.
    For the other bound of the first equation, we have
    \begin{equation}\label{eq:2th2}
    \begin{aligned}
        \norm{A(z)}_F &\geq \norm{\nabla_z \Phi(z)^\top \tilde \Phi^\top \tilde K^{-1}}_F - \bigO{\sqrt {N + d}/d} \\
        &\geq \norm{\nabla_z \Phi(z)^\top \tilde \Phi^\top}_F \lambda_{\min} \left(\tilde K^{-1} \right)-\bigO{\sqrt {N + d}/d} \\ &= \lambda_{\max}^{-1}(\tilde K) \mathcal I(z) - \bigO{\sqrt {N + d}/d},
    \end{aligned}
    \end{equation}
    where the first inequality is justified again by Lemma \ref{lemma:totalcentering}, and holds with probability at least $1 - \exp(-c \log^2 N) - k \exp(-c d)$ over $X$ and $V$.

    The second equation directly follows from Lemma \ref{lemma:evmaxfeatures} and Lemma \ref{lemma:evmincentered}.
\end{proof}

\subsection{Estimate of \texorpdfstring{$\norm{\Irf(z)}_F$}{Lg}}\label{app:estimateI}

The goal of this sub-Section is to estimate the value of the Frobenius norm of the Interaction Matrix $\norm{\Irf}_F$, defined in \eqref{eq:interactionmatrix}. To do so, we will estimate the squared $\ell_2$ norms of its columns. In particular, we define
\begin{equation}
    \mathcal I_x := \norm{\nabla_z \phi(V z)^\top \tilde \phi(V x)}_2.
\end{equation}
Thus, we have $\norm{\Irf}_F^2 = \sum_i \mathcal I^2_{x_i}$. In this section, with $x \sim P_X$ we refer to a generic input data, where we drop the index for compactness. Also, we will assume $z \sim P_X$, independent from the training set.

Recalling that
\begin{equation}
    \nabla_z \phi(V z)^\top =  V^\top \text{diag} \left( \phi' (Vz)\right),
\end{equation}
we have
\begin{equation}
    \mathcal I_x = \norm{V^\top \text{diag} \left( \phi' (Vz)\right) \tilde \phi (Vx)}_2.
\end{equation}

Along the proof, the following shorthands and notations will be useful
\begin{equation}
    \mathcal X_{-jl} = \sum_{m \neq j} V_{lm}x_{m}, \qquad \mathcal X_{l} = \sum_{m = 1}^d V_{lm}x_{m},
\end{equation}
\begin{equation}
    \mathcal Z_{-jl} = \sum_{m \neq j} V_{lm}z_{m}, \qquad \mathcal Z_{l} = \sum_{m = 1}^d V_{lm}z_{m},
\end{equation}
\begin{equation}\label{eq:Ij}
    I_j = \sum_{l = 1}^k V_{lj} \phi' \left( \mathcal Z_{l} \right) \tilde \phi \left( \mathcal X_{l} \right),
\end{equation}
which imply
\begin{equation}
    \mathcal I^2_x = \sum_{j = 1}^d I_j^2.
\end{equation}

It will also be convenient to use the following notation
\begin{equation}
    B_{-jl} := \left( \phi'' \left( \mathcal Z_{-jl} \right) \tilde \phi \left( \mathcal X_{-jl} \right) z_j + \phi' \left( \mathcal Z_{-jl} \right) \phi' \left( \mathcal X_{-jl} \right) x_j \right),
\end{equation}
\begin{equation}
    B_{jl} := \left( \phi'' \left( \mathcal Z_{l} \right)\tilde \phi \left(  \mathcal X_{l} \right) z_j + \phi' \left( \mathcal Z_{l} \right) \phi' \left( \mathcal X_{l} \right) x_j \right).
\end{equation}

For compactness, we will not necessarily re-introduce such quantities in the statements or the proofs of the lemmas in this section.

\begin{lemma}\label{lemma:utils}
    Indicating with $c$ a numerical constant, the following statements hold
    
    \begin{equation}
        \max_l \left| \tilde \phi(\mathcal X_{l}) \right| = \bigO{\log k},
    \end{equation}
    with probability at least $1 - \exp (- c \log^2 k)$, over $x$ and $V$;

    \begin{equation}
        \max_l \left| \tilde \phi(\mathcal X_{-jl}) \right| = \bigO{\log k},
    \end{equation}
     with probability at least $1 - \exp (-c \log^2 k)$, over $x$ and $\{V_{li}\}_{i \neq j, 1 \leq l \leq k}$;
     
    \begin{equation}
       \left| \sum_{l = 1}^k V_{lj} \phi' \left( \mathcal Z_{-jl} \right) \tilde \phi \left( \mathcal X_{-jl} \right) \right| = o \left(\frac{k}{d} \right),
    \end{equation}
    with probability at least $1 - \exp (-c \log^2 k)$, over $x$ and $V$.
\end{lemma}

\begin{proof}
    For the first result, as $\phi$ is Lipschitz and centered with respect to $x$, its sub-Gaussian norm (in the probability space of $x$) is $\bigO{\norm{V_{l:}}_2} = \bigO{1}$, where the last equality holds with probability at least $1 - e^{-c_0 d}$ over $V_{l:}$, because of Theorem 3.1.1 in \cite{vershynin2018high}. In particular, with a union bound over $V_{l:}$, we have $\bigO{\max_l \norm{V_{l:}}_2} = \bigO{1}$ with probability at least $1 - k e^{-c_0 d} \geq 1 - e^{-c_1 d}$, where the last inequality is justified by Assumption \ref{ass:dlarge}. Thus, by Lemma \ref{lemma:opnormlog}, $\max_l \left| \tilde \phi(\mathcal X_{l}) \right| = \bigO{\log k}$, with probability at least $1 - \exp (-c_1 d ) - \exp(- c_2 \log^2 k) \geq 1 - \exp(- c_3 \log^2 k) $, over $x$ and $V$. 

    Notice that, with the same argument, we can prove that $\max_l \left| \tilde \phi(\mathcal X_{-jl}) \right| = \bigO{\log k}$, with probability at least $1 - \exp(- c_4 \log^2 k)$, over $x$ and $\{V_{li}\}_{i \neq j, 1 \leq l \leq k}$. We will condition on such event for the rest of the proof.
    
    For the last statement, we have
    \begin{equation}
        \sum_{l = 1}^k V_{lj} \phi' \left( \mathcal Z_{-jl} \right) \tilde \phi \left( \mathcal X_{-jl} \right) =: V_{:j}^\top u,
    \end{equation}
    where $u$ is a vector independent from $V_{:j}$ such that
    \begin{equation}
        \norm{u}_2 \leq \sqrt{k} L \max_l \left| \tilde \phi \left( \mathcal X_{-jl} \right) \right| = \bigO{\log (k) \sqrt k}.
    \end{equation}
    Since $\subGnorm{V_{:j}} \leq C / \sqrt{d}$, we have
    \begin{equation}
        \P \left( \left| \sum_{l = 1}^k V_{lj} \phi' \left( \mathcal Z_{-jl} \right) \tilde \phi \left( \mathcal X_{-jl} \right) \right| > \frac{ \sqrt d}{\sqrt k /\log^2 k } \frac{k}{d} \right) < \exp (- c_5 \log^2 k),
    \end{equation}
    where the probability is intended over $V_{:j}$. Thus, using Assumption \ref{ass:overparam}, we readily get the desired result.
\end{proof}

\begin{lemma}\label{lemma:notsketchy}
We have that
    \begin{equation}
        \subGnorm{ \E_{x} \left[ \phi' \left( \mathcal X_{-jl} \right) x_j \right]} = \bigO{\frac{1}{\sqrt d}},
    \end{equation}
    where the norm is intended over the probability space of $\{ V_{li}\}_{i \neq j}$.
\end{lemma}
\begin{proof}
    Let's define an auxiliary random variable $\rho \sim \mathcal N(0, 1/d)$, independent from any other random variable in the problem. Let's also define $V'_{l:}$ as the copy of $V_{l:}$ where we replaced the $j$-th component with $\rho$. Notice that $V'_{l:}$ is a Gaussian vector with covariance $I/d$. Let's also define $\mathcal X'_l$ in the same way, i.e., replacing $V_{lj}$ with $\rho$ in its expression. All the tail norms $\norm{\cdot}_{\psi_r}$ in the proof will be referred to the probability space of $V'_{l:}$.
    
    Since $\phi'$ is a Lipschitz function, we can write
    \begin{equation}\label{eq:noprobstat}
        \left |\E_{x} \left[ \phi' \left( \mathcal X_{-jl} \right) x_j \right] - \E_{x} \left[ \phi' \left( \mathcal X'_{l} \right) x_j \right] \right| \leq L_1 \E_{x} \left[ | \rho| x_j^2 \right] \leq C | \rho|,
    \end{equation}
    where we used the fact that $x_j$ is sub-Gaussian in the last inequality, and where $C$ is a numerical constant.
    Let's now look at the function
    \begin{equation}
        \varphi(v) := \E_{x} \left[ \phi' \left( v^\top x \right) x_j \right].
    \end{equation}
    We have
    \begin{equation}
    \begin{aligned}
        \left |\varphi(v + v') - \varphi(v) \right| &= \left|\E_{x} \left[ \left( \phi' \left( (v+v')^\top x \right) - \phi' \left( v^\top x \right) \right) x_j \right]\right| \\
        &\leq \E_{x} \left[ \left| \left( \phi' \left( (v+v')^\top x \right) - \phi' \left( v^\top x \right) \right)\right| | x_j | \right] \\
        &\leq \E_{x} \left[ L_1 |v'^\top x|  | x_j | \right].
    \end{aligned}
    \end{equation}
    Since $x$ is sub-Gaussian (and therefore also $x_j$ is) we have that $|v'^\top x|  | x_j |$ is sub-exponential with norm (with respect to the probability space of $x$) upper bounded by $C_1 \norm{v'}_2$. Then, we have
    \begin{equation}
        \left |\varphi(v + v') - \varphi(v) \right| \leq \E_{x} \left[ L_1 |v'^\top x|  | x_j | \right] \leq C_2 \norm{v'}_2.
    \end{equation}
    This implies that $\norm{\varphi}_{\text{Lip}} \leq C_2$.

    As $V'_{l:}$ is Gaussian (and hence Lipschitz concentrated) with covariance $I/d$, we can write
    \begin{equation}
        \subGnorm{\varphi(V'_{l:}) - \E_{V'_{l:}} \left[ \varphi(V'_{l:}) \right]} = \bigO{\frac{1}{\sqrt{d}}}.
    \end{equation}
Furthermore,
    \begin{equation}
         \E_{V'_{l:}} \left[ \varphi(V'_{l:}) \right] = \E_{x} \left[ x_j \E_{V'_{l:}} \left[ \phi' \left( {V'_{l:}}^\top x \right) \right] \right] = \E_{x} \left[ x_j \right] \E_{\rho'} \left[ \phi' \left( \rho' \right) \right] = 0,
    \end{equation}
    where $\rho'$ is a standard Gaussian distribution independent of $x$. Note that the second step holds since $\norm{x}_2 = \sqrt d$, and the last step is justified by $\E[x] = 0$.
    
    This last equation implies
    \begin{equation}
    \begin{aligned}
        \subGnorm{\E_{x} \left[ \phi' \left( \mathcal X_{-jl} \right) x_j \right]} &\leq 
        \subGnorm{\E_{x} \left[ \phi' \left( \mathcal X_{-jl} \right) x_j \right] - \E_{x} \left[ \phi' \left( \mathcal X'_{l} \right) x_j \right]} + \subGnorm{\E_{x} \left[ \phi' \left( \mathcal X'_{l} \right) x_j \right]} \\
        &\leq C \subGnorm{\rho} + \subGnorm{\varphi (V'_{l:})} = \bigO{\frac{1}{\sqrt d}},
    \end{aligned}
    \end{equation}
where the first term in the last line is justified by \eqref{eq:noprobstat}. In fact, the random variable on the LHS is (with probability one) upper bounded by $C |\rho|$. Then, the bound subsists also between their tail norms.

    Thus,
    \begin{equation}
        \P_{\{ V_{li}\}_{i \neq j}} \left( \left| \E_{x} \left[ \phi' \left( \mathcal X_{-jl} \right) x_j \right] \right| > t \right) = \P_{V'_{l:}} \left( \left| \E_{x} \left[ \phi' \left( \mathcal X_{-jl} \right) x_j \right] \right| > t \right) < 2 \exp \left( -cd\, t^2 \right),
    \end{equation}
where with $\P_{\{ V_{li}\}_{i \neq j}}$ we refer to the probability space of $\{ V_{li} \}_{i \neq j}$. Note that the first equality is true as the statement in the probability does not depend on the value of $\rho$. From this, we readily get the thesis.
\end{proof}

\begin{lemma}\label{lemma:Gapprox1}
We have that
    \begin{equation}
        \left| I_j - \sum_{l} V_{lj}^2 \left( \phi'' \left( \mathcal Z_{-jl} \right) \tilde \phi \left( \mathcal X_{-jl} \right) z_j + \phi' \left( \mathcal Z_{-jl} \right) \phi' \left( \mathcal X_{-jl} \right) x_j \right) \right| = o \left( \frac{k}{d} \right).
    \end{equation}
    with probability at least $1 - \exp (-c  \log^2 k)$ over $x$, $z$ and $V$, where $c$ is an absolute constant.
\end{lemma}

\begin{proof}

Exploiting Taylor's Theorem, and the fact that both $\phi'$ and $\phi''$ are Lipschitz, we can write
\begin{equation}\label{eq:taylor1}
    \phi' \left(\mathcal Z_l \right) = \phi' \left( \mathcal Z_{-jl} \right) + \phi'' \left( \mathcal Z_{-jl} \right) V_{lj} z_j + a_{1lj},
\end{equation}
\begin{equation}\label{eq:taylor2}
    \tilde \phi \left( \mathcal X_l \right) =  \phi \left( \mathcal X_{-jl} \right) + \phi' \left( \mathcal X_{-jl} \right) V_{lj} x_j + a_{2lj} - \E_x \left[ \phi \left( \mathcal X_{-jl} \right) + \phi' \left( \mathcal X_{-jl} \right) V_{lj} x_j + a_{2lj} \right],
\end{equation}
where $|a_{1lj}| = \bigO{V_{lj}^2 z_j^2}$ and $|a_{2lj}| = \bigO{V_{lj}^2 x_j^2}$. In the previous two equations, we expanded $\phi'$ around $\mathcal Z_{-jl}$ and $\phi$ around $\mathcal X_{-jl}$. For the rest of the proof (except from when expectations with respect to $x$ are taken) we condition on the events that both $|x_j|$ and $|z_j|$ are $\bigO{\log k}$. As they are sub-Gaussian, this happens with probability at least $1 - e^{- c_1\log^2 k}$ over $x$ and $z$.

Invoking Lemma \ref{lemma:alphatails}, and setting $t = k / d^{1/\alpha}$, we can write (for $\alpha \in (0, 2]$)
\begin{equation}
    \sum_{l = 1}^k |V_{lj}|^{2 / \alpha} = \bigO{\frac{k}{d^{1 / \alpha}}},
\end{equation}
with probability at least $1 - \exp (-c_2 k^{\alpha/2} / \log k)$. This is true since $\E_\rho \left[ \rho^{ 2 / \alpha} \right] = C_\alpha$, where $\rho$ is a standard Gaussian random variable. Thus, with probability at least $1 - \exp (-c_3 k^{1/5} / \log k)$ over $V$, we jointly have
\begin{equation}
    \sum_{l = 1}^k |V_{lj}|^2 = \bigO{\frac{k}{d}}, \qquad \sum_{l = 1}^k |V_{lj}|^3 = \bigO{\frac{k}{d^{3/2}}}, \qquad \sum_{l = 1}^k |V_{lj}|^5 = \bigO{\frac{k}{d^{5/2}}}.
\end{equation}
We will condition on these events too. It will also be convenient to condition on the high probability events described by Lemma \ref{lemma:utils}. Thus, our discussion is now restricted to a probability space over $x$, $z$ and $V$ of measure at least $1 - e^{- c_4 \log^2 k} - e^{- c_5 d}$.

We are now ready to estimate the cross terms of the product $V_{lj} \phi' \left( \mathcal Z_{l} \right) \tilde \phi \left( \mathcal X_{l} \right)$, exploiting the expansions in \eqref{eq:taylor1} and \eqref{eq:taylor2}. Let's control the sums over $l$ of all of them separately.

\begin{enumerate}
    \item[(a)] $\phi' \left( \mathcal Z_l \right) \left( a_{2lj} - \E_x \left[ a_{2lj} \right] \right)$
    \begin{equation}
    \begin{aligned}
    & \left| \sum_{l = 1}^k V_{lj} \phi' \left( \mathcal Z_l \right) \left( a_{2lj} - \E_x \left[ a_{2lj} \right] \right) \right| \leq C_1 \sum_{l = 1}^k |V_{lj}| L \left( |a_{2lj}| + \E_x \left[ |a_{2lj}| \right] \right) \\
    & \qquad \leq C_2 \left( x_j^2 + 1 \right) \sum_{l = 1}^k |V_{lj}|^3 = \bigO{\frac{k \log^2  k}{d^{3/2}}} = o\left(\frac{k}{d}\right),
    \end{aligned}
    \end{equation}
    where the second inequality is justified by the fact that $x_j$ is sub-Gaussian, and the last equality by Assumption \ref{ass:dlarge}.

    \item[(b)] $\tilde \phi \left( \mathcal X_{l} \right) a_{1jl}$
    \begin{equation}
    \begin{aligned}
    & \left| \sum_{l = 1}^k V_{lj}  \tilde \phi \left( \mathcal X_{l} \right) a_{1jl} \right| \leq C_3 z_j^2 \sum_{l = 1}^k |V_{lj}|^3  \left| \tilde \phi \left( \mathcal X_{l} \right) \right| \leq C_3 z_j^2 \max_l \left| \tilde \phi \left( \mathcal X_{l} \right) \right|  \sum_{l = 1}^k |V_{lj}|^3 \\
    &\qquad = \bigO{\log^2 (k) \log (k) \frac{k}{d^{3/2}}} = o\left( \frac{k}{d}\right),
    \end{aligned}
    \end{equation}
    where the second to last equality follows from Lemma \ref{lemma:utils}, and the last from Assumption \ref{ass:dlarge}.
    
    \item[(c)] $a_{1lj} \left( a_{2lj} - \E_x \left[ a_{2lj} \right] \right)$ (as we counted it twice)
    \begin{equation}
    \left |\sum_{l = 1}^k V_{lj} a_{1lj} \left( a_{2lj} - \E_x \left[ a_{2lj} \right] \right) \right| \leq C_4 \left( x_j^2 + 1 \right) z^2_j \sum_{l = 1}^k |V_{lj}|^5 = \bigO{\log^4(k) \frac{k}{d^{5/2}}} = o\left(\frac{k}{d}\right).
    \end{equation}

    \item[(d)] $\phi' \left( \mathcal Z_{-jl} \right) \tilde \phi \left( \mathcal X_{-jl} \right)$
    \begin{equation}
        \left| \sum_{l = 1}^k V_{lj} \phi' \left( \mathcal Z_{-jl} \right) \tilde \phi \left( \mathcal X_{-jl} \right)\right| = o \left(\frac{k}{d}\right),
    \end{equation}
    as it directly follows from Lemma \ref{lemma:utils}.

    \item[(e)] $\left( \phi' \left( \mathcal Z_{-jl} \right) + \phi'' \left( \mathcal Z_{-jl} \right) V_{lj} z_j \right) \E_x \left[ \phi' \left( \mathcal X_{-jl} \right) V_{lj} x_j \right]$
    \begin{equation}
    \begin{aligned}
        & \left| \sum_{l = 1}^k V_{lj} \left( \phi' \left( \mathcal Z_{-jl} \right) + \phi'' \left( \mathcal Z_{-jl} \right) V_{lj} z_j \right) V_{lj} \E_{x} \left[ \phi' \left( \mathcal X_{-jl} \right) x_j \right] \right| \\
        & \qquad \leq  \max_l \left| \phi' \left( \mathcal Z_{-jl} \right) + \phi'' \left( \mathcal Z_{-jl} \right) V_{lj} z_j \right| \max_l \left| \E_{x} \left[ \phi' \left( \mathcal X_{-jl} \right) x_j \right] \right| \sum_{l = 1}^k |V_{lj}|^2 \\
        & \qquad \leq \left( L + L_1 \max_l \left|V_{lj}\right| |z_j| \right) \max_l \left| \E_{x} \left[ \phi' \left( \mathcal X_{-jl} \right) x_j \right] \right| \sum_{l = 1}^k |V_{lj}|^2 \\
        & \qquad = \bigO{\left( 1 + \frac{\log^2 k}{\sqrt d}\right) \frac{\log k}{\sqrt d}} \bigO{\frac{k}{d}} = o \left(\frac{k}{d}\right),
    \end{aligned}
    \end{equation}
    where the second to last equality comes from the fact that both $\subGnorm{V_{lj}}$ and $\subGnorm{ \E_{x} \left[ \phi' \left( \mathcal X_{-jl} \right) x_j \right]}$ are $\bigO{\frac{1}{\sqrt d}}$ (see Lemma \ref{lemma:notsketchy}) and holds with probability at least $1 - e^{- c_6 \log^2 k}$ over $V$. The last equality is true for Assumption \ref{ass:dlarge}.
    
    \item[(f)] $\phi'' \left( \mathcal Z_{-jl} \right) V_{lj} z_j \phi' \left( \mathcal X_{-jl} \right) V_{lj} x_j$ 
    As $\phi$ and $\phi'$ are Lipschitz
    \begin{equation}
        \left| \sum_{l = 1}^k V_{lj} \phi'' \left( \mathcal Z_{-jl} \right) V_{lj} z_j \phi' \left( \mathcal X_{-jl} \right) V_{lj} x_j \right| \leq L L_1 |x_j z_j| \sum_{l = 1}^k |V_{lj}|^3 = \bigO{\log^2 k \frac{k}{d^{3/2}}} = o\left( \frac{k}{d}\right).
    \end{equation}
\end{enumerate}

Then, we are left only with two cross terms, and we can write
\begin{equation}
        I_j = \sum_{l} V_{lj}^2 \left( \phi'' \left( \mathcal Z_{-jl} \right) \tilde \phi \left( \mathcal X_{-jl} \right) z_j + \phi' \left( \mathcal Z_{-jl} \right) \phi' \left( \mathcal X_{-jl} \right) x_j \right) + \textup{rest},
\end{equation}
where $|\textup{rest}|$ can be upper-bounded by the sum of the terms (a)-(f), which gives

\begin{equation}
    \left| I_j - \sum_{l} V_{lj}^2 \left( \phi'' \left( \mathcal Z_{-jl} \right) \tilde \phi \left( \mathcal X_{-jl} \right) z_j + \phi' \left( \mathcal Z_{-jl} \right) \phi' \left( \mathcal X_{-jl} \right) x_j \right) \right| = o \left( \frac{k}{d}\right)
\end{equation}
with probability at least $1 - \exp (-c_5  d) - \exp (-c_7  \log^2 k)$ over $x$, $z$ and $V$. Hence, the thesis follows after using Assumption \ref{ass:dlarge}.
\end{proof}

\begin{lemma}\label{lemma:Gapprox2}
    We have that
    \begin{equation}
        \left| \sum_{l = 1}^k V_{lj}^2 B_{-jl} - \frac{1}{d} \sum_{l= 1}^k B_{-jl} \right| = o \left( \frac{k}{d} \right),
    \end{equation}
    with probability at least $1 - \exp(-c \log^2 k)$ over $x$, $z$ and $V$, where $c$ is an absolute constant.
\end{lemma}
\begin{proof}
    We have
    \begin{equation}
        \sum_{l = 1}^k V_{lj}^2 B_{-jl} = \sum_{l = 1}^k \left( V_{lj}^2 - \E_V \left[ V_{lj}^2 \right] \right) B_{-jl} + \frac{1}{d} \sum_{l= 1}^k B_{-jl},
    \end{equation}
    since $\E_V \left[ V_{lj}^2 \right] = 1/d$ for all $l, j$.

    We call $v_l := V_{lj}^2 - \E_V \left[ V_{lj}^2 \right]$. Then, $v=(v_1, \ldots, v_k)$ is a mean-0 vector with independent and sub-exponential entries with norm $\bigO{1/d}$. This implies $\subEnorm{v} = \bigO{1/d}$. Thus,
    \begin{equation}\label{eq:subEvectorforV}
        \P_v \left( |v^\top u| > t \right) < 2 \exp \left(- c \frac{d t}{\norm{u}_2} \right),
    \end{equation}
    for every vector $u$ independent from $v$. If we define $u_l := B_{-jl}$, we have
    \begin{equation}
        \norm{u}_2 \leq \sqrt k \left( \max_l |B_{-jl}|\right) \leq \sqrt k \left( L^2 |x_j| + L_1 |z_j| \max_l \left| \tilde \phi(\mathcal X_{-jl}) \right| \right) = \bigO{\log^2 (k) \sqrt k},
    \end{equation}
    where the last equality comes from the fact that $x_j$ and $z_j$ are sub-Gaussian, and from the second statement in Lemma \ref{lemma:utils}. This holds with probability at least $1 - e^{-c_1 d} - e^{-c_2 \log^2 k}$ over $x$, $z$ and $\{V_{li}\}_{i \neq j, 1 \leq l \leq k}$, and we will condition on this event.

    Then, merging this last result with \eqref{eq:subEvectorforV} and setting $t = \log^4 k \sqrt k / d$ we have
    \begin{equation}
        \P_v \left( |v^\top u| > \frac{\log^4 k \sqrt k}{d} \right) < 2 \exp \left(- c_3 \log^2 k \right).
    \end{equation}
    
    Performing a union bound gives the thesis.
\end{proof}

\begin{lemma}\label{lemma:Gapprox3}
    We have that
    \begin{equation}
    \left| \sum_{l = 1}^k B_{-jl} - \sum_{l = 1}^k B_{jl} \right| = o(k),
    \end{equation}
    with probability at least at least $1 - \exp(-c \log^2 k)$ over $x$, $z$ and $V$, where $c$ is an absolute constant.
\end{lemma}
\begin{proof}
    Let's condition on
    \begin{equation}
        |x_j| = \bigO{\log k}, \qquad |z_j| = \bigO{\log k}, \qquad \max_l |V_{lj}| = \bigO{\frac{\log k}{\sqrt d}}.
    \end{equation}
    These events happen with probability at least $1 - \exp(-c \log^2 k)$, as $\subGnorm{x_j} = \bigO{1}$, $\subGnorm{z_j} = \bigO{1}$ and $\subGnorm{V_{jl}} = \bigO{1/\sqrt{d}}$. The probability is over $x$, $z$ and $V$. Also, we have $\max_l \left| \tilde \phi(\mathcal X_{-jl}) \right| = \bigO{\log k}$ with probability at least $1 - e^{-c_1 \log^2 k}$ over $x$, $z$ and $V$, for the second statement in Lemma \ref{lemma:utils}. Let's condition on these high probability events for the rest of the proof, whenever we do not take expectations with respect to $x$.

    Exploiting that $\phi$, $\phi'$ and $\phi''$ are all Lipschitz,
    \begin{equation}
         \max_l \left |\phi'' \left( \mathcal Z_{-jl} \right) -  \phi'' \left( \mathcal Z_{l} \right) \right| = \bigO{\frac{\log^2 k}{\sqrt{d}}},
    \end{equation}
    \begin{equation}
        \max_l \left( \left| \phi \left( \mathcal X_{-jl} \right) - \phi \left( \mathcal X_{-jl} \right) \right| + \left|\E_x \left[ \phi \left( \mathcal X_{-jl} \right) - \phi \left( \mathcal X_{-jl} \right) \right] \right| \right) \leq \bigO{\frac{\log^2 k}{\sqrt d}} + \max_l |V_{lj}| \E_x[|x_j|] = \bigO{\frac{\log^2 k}{\sqrt d}},
    \end{equation}
    \begin{equation}
        \max_l \left |\phi' \left( \mathcal Z_{-jl} \right) -  \phi' \left( \mathcal Z_{l} \right) \right| = \bigO{\frac{\log^2 k}{\sqrt{d}}},
    \end{equation}
    \begin{equation}
        \max_l \left| \phi' \left( \mathcal X_{-jl} \right) - \phi' \left( \mathcal X_{-jl} \right) \right| = \bigO{\frac{\log^2 k}{\sqrt{d}}}.
    \end{equation}
    Then, controlling all the cross terms, we get
    \begin{equation}
        \max_l |B_{-jl} - B_{jl}| = \bigO{\frac{\log^4 k}{\sqrt d}},
    \end{equation}
    where this scaling is justified by the leading term
    \begin{equation}
        \max_l \left |\phi'' \left( \mathcal Z_{-jl} \right) -  \phi'' \left( \mathcal Z_{l} \right) \right| \max_l \left| \tilde \phi(\mathcal X_{-jl}) \right| |z_j| = \bigO{\frac{\log^2 k}{\sqrt{d}} \log k \log k}.
    \end{equation}

    Thus, we can conclude
    \begin{equation}
        \left| \sum_{l = 1}^k B_{-jl} - \sum_{l = 1}^k B_{jl} \right| \leq k \max_l |B_{-jl} - B_{jl}| = \bigO{\frac{k \log^4 k}{\sqrt d}} = o(k),
    \end{equation}
    because of Assumption \ref{ass:dlarge}. This gives the thesis.
\end{proof}

\begin{lemma}\label{lemma:Gapprox4}
    We have that
    \begin{equation}
    \left| \sum_{l = 1}^k B_{jl} - k \, \E_V [B_{j1}] \right| = o(k),
    \end{equation}
    with probability at least $1 - \exp (-c \log^2 k)$ over $x$, $z$ and $V$.
\end{lemma}

\begin{proof}
    Since $x_j$ and $z_j$ are sub-Gaussian, we know that $|x_j|,|z_j|=\bigO{\log k}$ with probability at least $1 - \exp (-c \log^2 k)$ over $x$ and $z$. We will condition on such event for the rest of the proof. Note that, once we fix the values of $x$ and $z$, the $B_{jl}$'s are i.i.d. random variables in the probability space of $V$. Thus,
    \begin{equation}
        \sum_{l = 1}^k B_{jl} = k \, \E_V [B_{j1}] + \sum_{l = 1}^k \left( B_{jl} - \E_V [B_{jl}] \right).
    \end{equation}
    Until the end of the proof, we only refer to the probability space of $V$.
    
    We have
    \begin{equation}\label{eq:phitsubGv}
        \subGnorm{\tilde \phi \left( \mathcal X_{l} \right)} \leq \subGnorm{\phi \left( \mathcal X_{l} \right)} + \subGnorm{\E_x \left[ \phi \left( \mathcal X_{l} \right) \right]} \leq C \norm{x}_2 \frac{1}{\sqrt d} + \E_x \left[ \subGnorm{\phi \left( \mathcal X_{l} \right)} \right] \leq 2C,
    \end{equation}
    where the second inequality is true because $\subGnorm{x^\top V_{l:}} \leq C_1 \norm{x} \frac{1}{\sqrt d}$ and because $\phi$ is Lipschitz.

    Since $\phi$ and $\phi'$ are both Lipschitz, we have that also $\subGnorm{\phi'' \left( \mathcal Z_{l} \right)}$, $\subGnorm{\phi' \left( \mathcal Z_{l} \right)}$ and $\subGnorm{\phi' \left( \mathcal X_{l} \right)}$ are upper bounded by a numerical constant. Thus, we have that
    \begin{equation}
        \subEnorm{B_{jl}} \leq \log k.
    \end{equation}

    An application of Bernstein inequality (cf. Theorem 2.8.1. in \cite{vershynin2018high}) gives that
    \begin{equation}
        \P \left( \left | \sum_{l = 1}^k \left( B_{jl} - \E_V [B_{jl}] \right) \right| > \sqrt{k} \log^2 k \right) < \exp (- c_1 \log^2 k),
    \end{equation}
and the thesis readily follows.
\end{proof}

\begin{lemma}\label{lemma:Gapprox5}
    Indicating with $\rho$ a standard Gaussian random variable, we have that
    \begin{equation}
        \left| \E_V [B_{j1}] - x_j \E_{\rho}^2 [\phi' (\rho)] \right| = o(1),
    \end{equation}
    with probability at least $1 - \exp (- c \log^2 k)$ over $x$ and $z$.
\end{lemma}

\begin{proof}
    Let's define the vector $z'$ as follows
    \begin{equation}
        z' = \frac{ \sqrt d \left( I - \frac{xx^\top}{d} \right) z }{\norm{ \left( I - \frac{xx^\top}{d} \right) z}_2}.
    \end{equation}

    Notice that, by construction, $x^\top z' = 0$, and $\norm{z'}_2 = \sqrt d$. Also, consider a vector $y$ orthogonal to both $z$ and $x$. Then, a fast computation returns $y^\top z' = 0$. This means that $z'$ is the vector on the $\sqrt d$-sphere, lying on the same plane of $z$ and $x$, orthogonal to $x$. Thus, we can easily compute 
    \begin{equation}
        \frac{\left| z^\top z' \right|}{d}= \sqrt{1 - \left( \frac{x^\top z}{d} \right)^2} \geq 1 - \left( \frac{x^\top z}{d} \right)^2,
    \end{equation}
    where the last inequality derives from $\sqrt{1 - a} \geq 1 - a$ for $a \in [0, 1]$. Then,
    \begin{equation}
        \norm{z - z'}_2^2 = \norm{z}_2^2 + \norm{z}_2^2 - 2 z^\top z' \leq 2d \left(1 - \left( 1 - \left( \frac{x^\top z}{d} \right)^2\right)\right) = 2 \frac{\left( x^\top z\right)^2}{d}.
    \end{equation}

    As $x$ and $z$ are both sub-Gaussian, mean-0 vectors, with $\ell_2$ norm equal to $\sqrt d$, we have that
    \begin{equation}\label{eq:normissubG}
        \P \left( \norm{z - z'}_2 > t \right) \leq \P \left( | x^\top z | > \sqrt d t / \sqrt 2 \right) < \exp (- c t^2),
    \end{equation}
    where $c$ is an absolute constant. Here the probability is referred to the space of $x$, for a fixed $z$. Thus, $\norm{z - z'}_2$ is sub-Gaussian.

    Let's now introduce the shorthands $v := V_{1j}$ and $\mathcal Z_1' = v^\top z'$. We have
    \begin{equation}\label{eq:exp1g7}
        \left|  \E_v \left[ \phi''(\mathcal Z_1) \phi(\mathcal X_1) \right] - \E_v \left[ \phi''(\mathcal Z'_1) \phi (\mathcal X_1)\right]   \right| \leq L_2 \E_v \left[ | v^\top (z-z')| | \phi (\mathcal X_1)| \right].
    \end{equation}
    Since $\phi$ is Lipschitz, $\phi (\mathcal X_1)$ is sub-Gaussian with respect to the probability space of $v$. Since $z$ and $z'$ are independent from $v$, we deduce that
    \begin{equation}
        \subGnorm{v^\top (z-z')} \leq \norm{z - z'}_2 \frac{C}{\sqrt d} = \bigO{\frac{\log k}{\sqrt d}},
    \end{equation}
    with probability at least $1 - \exp (- c \log^2 k)$ over $x$. If we condition on such event, this implies that the expectation in \eqref{eq:exp1g7} can be bounded as
    \begin{equation}\label{eq:endexp1g7}
        \E_v \left[ | v^\top (z-z')| | \phi (\mathcal X_1)| \right] = \bigO{\frac{\log k}{\sqrt d}}.
    \end{equation}
    This is true since we are computing the first moment of a sub-exponential random variable.

    A similar computation gives
    \begin{equation}\label{eq:exp2g7}
    \begin{aligned}
        & \left|  \E_v \left[ \phi''(\mathcal Z_1) \E_x[\phi(\mathcal X_1)] \right] - \E_{vx}\left[ \phi''(\mathcal Z'_1) \phi (\mathcal X_1)\right]  \right| \\
        & \qquad \leq L_2 \E_x \E_v \left[ | v^\top (z-z')| | \phi (\mathcal X_1)| \right] \leq C \frac{\E_x \left[ \norm{z - z'}_2\right]}{\sqrt d} = \bigO{\frac{1}{\sqrt d}},
    \end{aligned}
    \end{equation}
    where the last step is a direct consequence of \eqref{eq:normissubG}.
    Notice that this, together with the previous argument between \eqref{eq:exp1g7} and \eqref{eq:endexp1g7}, gives
    \begin{equation}\label{eq:exp1totalg7}
         \left|  \E_v \left[ \phi''(\mathcal Z_1) \tilde \phi(\mathcal X_1) \right] - \left(  \E_{v} \left[ \phi''(\mathcal Z'_1) \phi (\mathcal X_1)\right] - \E_{vx} \left[ \phi''(\mathcal Z'_1) \phi (\mathcal X_1)\right] \right) \right| = \bigO{\frac{\log k}{d}}.
    \end{equation}

    With an equivalent argument of the one between \eqref{eq:exp1g7} and \eqref{eq:endexp1g7}, we can also show that
    \begin{equation}\label{eq:exp3g7}
        \left|  \E_v \left[ \phi'(\mathcal Z_1) \phi'(\mathcal X_1) \right] - \E_v \left[ \phi'(\mathcal Z'_1) \phi' (\mathcal X_1)\right]  \right| = \bigO{\frac{\log k}{\sqrt{d}}},
    \end{equation}
    with probability at least $1 - \exp (- c \log^2 k)$ over $x$.

    At this point, since $x^\top z' = 0$, we have have that (in the probability space of $V$), $\mathcal X_1$ and $\mathcal Z'_1$ can be treated as independent standard Gaussian random variables $\rho_1$ and $\rho_2$. We therefore have
    \begin{equation}\label{eq:gaussexp1}
        \E_v \left[ \phi''(\mathcal Z'_1) \phi (\mathcal X_1)\right] = \E_{\rho_1 \rho_2} \left[ \phi''(\rho_1) \phi (\rho_2)\right] = \E_{\rho_1}  \left[ \phi''(\rho_1)\right] \E_{\rho_2}  \left[ \phi(\rho_2)\right],
    \end{equation}
    \begin{equation}\label{eq:gaussexp2}
        \E_{vx} \left[ \phi''(\mathcal Z'_1) \phi (\mathcal X_1)\right] = \E_x \left[ \E_{\rho_1 \rho_2} \left[ \phi''(\rho_1) \phi (\rho_2)\right] \right] = \E_{\rho_1}  \left[ \phi''(\rho_1)\right] \E_{\rho_2}  \left[ \phi(\rho_2)\right],
    \end{equation}
    \begin{equation}\label{eq:gaussexp3}
        \E_v \left[ \phi'(\mathcal Z'_1) \phi' (\mathcal X_1) \right] = \E_{\rho_1 \rho_2} \left[ \phi'(\rho_1) \phi' (\rho_2)\right] = \E_{\rho_1}^2 \left[ \phi'(\rho_1) \right].
    \end{equation}

    Now, merging \eqref{eq:exp1totalg7} with \eqref{eq:gaussexp1} and \eqref{eq:gaussexp2}, we get
    \begin{equation} 
        \left|  \E_v \left[ \phi''(\mathcal Z_1) \tilde \phi(\mathcal X_1) \right] - ( \E_{\rho_1}  \left[ \phi''(\rho_1)\right] \E_{\rho_2}  \left[ \phi(\rho_2)\right] -  \E_{\rho_1}  \left[ \phi''(\rho_1)\right] \E_{\rho_2}  \left[ \phi(\rho_2)\right]) \right| = \left|  \E_v \left[ \phi''(\mathcal Z_1) \phi(\mathcal X_1) \right] \right| = \bigO{\frac{\log k}{\sqrt{d}}}.
    \end{equation}
    While merging \eqref{eq:exp3g7} with \eqref{eq:gaussexp3} returns
    \begin{equation}
        \left|  \E_v \left[ \phi'(\mathcal Z_1) \phi'(\mathcal X_1) \right] -  \E_{\rho_1}^2 \left[ \phi'(\rho_1) \right] \right| = \bigO{\frac{\log k}{\sqrt{d}}}.
    \end{equation}
    The previous statements hold with probability at least $1 - \exp (- c_1 \log^2 k)$ over $x$.

    Thus, conditioning on $|x_j| = \bigO{\log k}$ and $|z_j| = \bigO{\log k}$, we get
    \begin{equation}
    \left| \E_V [B_{j1}] - x_j \E_{\rho}^2 [\phi' (\rho)] \right| = \bigO{\frac{\log^2 k}{\sqrt{d}}} = o(1),
    \end{equation}
    with probability at least $1 - \exp (- c_2 \log^2 k)$ over $x$ and $z$.
    This readily gives the thesis.
\end{proof}

At this point, we are ready to prove Theorem \ref{thm:int}.

\begin{proof}[Proof of Theorem \ref{thm:int}]
    Let's focus on $\mathcal I^2_{x_1} = \sum_{j = 1}^d I_j ^2$, where with a slight abuse of notation we now refer to $I_j$ as the quantity defined in \eqref{eq:Ij}, with the variable $x_1$ instead of $x$. All the results from the previous lemmas hold as $x_1 \sim P_X$.
    We have, by triangle inequality,
    \begin{equation}
    \begin{aligned}
        &\left| I_j - x_j \E_{\rho}^2 [\phi' (\rho)] \frac{k}{d} \right| \leq \left| I_j - \sum_{l = 1}^k V_{lj}^2 B_{-jl}\right| + \left| \sum_{l = 1}^k V_{lj}^2 B_{-jl} - \frac{1}{d} \sum_{l= 1}^k B_{-jl} \right| \\
        &\qquad + \frac{1}{d} \left|\sum_{l= 1}^k B_{-jl} - \sum_{l = 1}^k B_{jl} \right| + \frac{1}{d} \left| \sum_{l = 1}^k B_{jl} - k \, \E_V [B_{j1}] \right| + \frac{k}{d} \left| \E_V [B_{j1}] - x_j \E_{\rho}^2 [\phi' (\rho)] \right| \\
        & = o\left( \frac{k}{d} \right) + o\left( \frac{k}{d} \right) + \frac{1}{d} o \left( k \right) + \frac{1}{d} o\left( k \right) + \frac{k}{d} o\left( 1 \right) = o\left( \frac{k}{d} \right),
    \end{aligned}
    \end{equation}
    where the five terms in the third line are justified respectively by Lemmas \ref{lemma:Gapprox1}, \ref{lemma:Gapprox2}, \ref{lemma:Gapprox3}, \ref{lemma:Gapprox4} and \ref{lemma:Gapprox5}, and jointly hold with probability at least $1 - \exp (- c_2 \log^2 k)$ over $x_1$ and $z$ and $V$.
    
    This holds for every $1 \leq j \leq d$ with probability at least $1 - d \exp (- c_2 \log^2 k) \geq 1 - \exp (- c_3 \log^2 k)$. Thus, we can write
    \begin{equation}
        I_{x_1} \geq \sqrt{\sum_{j = 1}^d \left( |x_j| \E_{\rho}^2 [\phi' (\rho)] \frac{k}{d} - o\left( \frac{k}{d} \right) \right)^2} \geq \sqrt{ \E_{\rho}^4 [\phi' (\rho)] \frac{k^2}{d^2} \sum_{j = 1}^d x_j^2  - o\left( \frac{k^2}{d} \right)  - o \left(  \frac{k^2 \sum_j |x_j|}{d^2}\right)}.
    \end{equation}
    By Cauchy-Schwartz, we have that $\sum_j |x_j| \leq \sqrt{d \sum_j x_j^2} = d$. If $\E_{\rho} [\phi' (\rho)] \neq 0$, we have
    \begin{equation}
        I_{x_1} \geq \sqrt{ \E_{\rho}^4 [\phi' (\rho)] \frac{k^2}{d}  - o\left( \frac{k^2}{d} \right)} = \E_{\rho}^2 [\phi' (\rho)] \frac{k}{\sqrt {d}} \sqrt{1 - o(1)} = \E_{\rho}^2 [\phi' (\rho)] \frac{k}{\sqrt {d}} - o\left( \frac{k}{\sqrt {d}} \right).
    \end{equation}
    Following the same strategy, we can also prove that
    \begin{equation}
        I_{x_1} \leq  \E_{\rho}^2 [\phi' (\rho)] \frac{k}{\sqrt {d}} + o\left( \frac{k}{\sqrt {d}} \right),
    \end{equation}
    which finally gives
    \begin{equation}
        \left| I_{x_1} -  \E_{\rho}^2 [\phi' (\rho)] \frac{k}{\sqrt {d}} \right| = o\left( \frac{k}{\sqrt {d}} \right).
    \end{equation}
    Notice that we easily retrieve the same result also in the case $\E_{\rho} [\phi' (\rho)] = 0$.

    Performing now a union bound over all the data, and recalling that $\norm{\Irf}_F = \sqrt{\sum_{i = 1}^N I^2_{x_i}} \geq \sqrt N \min_i I_{x_i}$ we have
    \begin{equation}
        \norm{\Irf}_F \geq \E_{\rho}^2 [\phi' (\rho)] \frac{k \sqrt N}{\sqrt d} - o \left( \frac{k \sqrt N}{\sqrt d} \right),
    \end{equation}
    with probability at least $1 - N \exp (- c_3 \log^2 k) \geq 1 - \exp (- c_4 \log^2 k)$ over $X$ and $z$ and $V$. The same argument proves also the upper bound, and the thesis readily follows.
\end{proof}


\subsection{Proof of Theorem \ref{thm:rfnorobust}}\label{proof:rfnorobust}

First, we state and prove an auxiliary result upper bounding $\opnorm{A}$.

\begin{lemma}\label{lemma:opnormA}
    Let $A = \nabla_z \phi(V z)^\top \phi(X V^\top)^\top K^{-1}$. Then, we have
    \begin{equation}
        \opnorm{A} = \bigO{\frac{1}{\sqrt{d}}},
    \end{equation}
    with probability at least $1 - \exp(-c \log^2 N) - \exp(-c d)$ over $V$ and $X$, where $c$ is an absolute constant.
\end{lemma}

\begin{proof}
Recalling that $\nabla_z \phi(V z)^\top = V^\top \text{diag} \left( \phi' (Vz)\right)$, we have
\begin{equation}
\begin{aligned}
    \opnorm{A} &= \opnorm{\nabla_z \phi(V z)^\top \phi(X V^\top)^\top K^{-1}} \\
    &= \opnorm{V^\top \text{diag} \left( \phi' (Vz)\right) \phi(X V^\top)^\top K^{-1}} \\
    & \leq \opnorm{V} \opnorm{\text{diag} \left( \phi' (Vz)\right)} \opnorm{\Phi^+} \\
    & \leq \opnorm{V} L \, \evmin{K}^{-1/2}.
\end{aligned}
\end{equation}
where in the last line we use that $\phi$ is an $L$-Lipschitz function and that $K^{-1} = (\Phi^+)^\top (\Phi^+)$.
By Theorem 4.4.5 of \cite{vershynin2018high}, we have (as $d = o(k)$ by Assumption \ref{ass:overparam}) that
\begin{equation}\label{eq:opnV}
    \opnorm{V} = \bigO{1 + \sqrt{\frac{k}{d}}} = \bigO{\sqrt{\frac{k}{d}}},
\end{equation}
with probability at least $1 - \exp(-c d)$ over $V$. Lemma \ref{lemma:evminfeatures} readily gives $\evmin{K} = \Omega(k)$, with probability at least $1 - \exp(-c \log^2 N)$ over $V$ and $X$. Taking a union bound on these two events gives the thesis.
\end{proof}

\begin{proof}[Proof of Theorem \ref{thm:rfnorobust}]
    By Theorem \ref{thm:int}, we have $\norm{\Irf}_F \geq C_1 \frac{k \sqrt N}{\sqrt d}$, with probability at least $1 - \exp (- c \log^2 k)$, which implies, by Theorem \ref{thm:Az},
    \begin{equation}
        \norm{A(z)}_F \geq C_2 \frac{d}{kN} \frac{k \sqrt N}{\sqrt d} - C \sqrt{N + d}/d = C_2 \frac{\sqrt d}{\sqrt N} - C \sqrt{N + d}/d \geq C_3 \frac{\sqrt d}{\sqrt N},
    \end{equation}
where the last inequality follows from Assumption \ref{ass:dlarge}, and the first holds with probability at least $1 - \exp(-c \log^2 N) - k \exp(-c d)$.
    This also means that
    \begin{equation}
        \mathcal S(z) \geq C_4 \sqrt d \norm{A(z)}_F
    \end{equation}
    holds with probability at least $1 - \exp \left( -c d^2 / N \right)$, because of Theorem \ref{thm:noiselowerbound} and Lemma \ref{lemma:opnormA}.
    Thus, we can conclude that
    \begin{equation}
        \mathcal S(z) \geq C_4  \norm{z} \norm{A(z)}_F \geq C_5 \frac{d}{\sqrt N} \geq C_6 \sqrt[6]{N},
    \end{equation}
    where the last inequality is a consequence of Assumption \ref{ass:dlarge}.
\end{proof}

\section{Proofs for NTK Regression}

In this section, we will indicate with $X \in \R^{N \times d}$ the data matrix, such that its rows are sampled independently from $P_X$ (see Assumption \ref{ass:datadist}). We will indicate with $W \in \R^{k \times d}$ the first half of the weight matrix at initialization, such that $W_{ij} \distas{}_{\rm i.i.d.}\mathcal{N}(0, 1/d)$. The feature vector will be therefore be (see \eqref{eq:ntkfeaturemap})
\begin{equation}
    \fntk(x) = \left[ x \otimes \phi'(W x), - x \otimes \phi'(W x) \right]
\end{equation}
where $\phi$ is the activation function, applied component-wise to the pre-activations $Wx$.

Notice that, for compactness, we dropped the subscripts \enquote{NTK} from these quantities, as this section will only treat the proofs related to Section \ref{sec:ntk}. Always for the sake of compactness, we will not re-introduce such quantities in the statements or the proofs of the following Lemmas.

\begin{lemma}\label{lemma:evminntk}
We have that
\begin{equation}
    \evmin{K} = \Omega(kd),
\end{equation}
with probability at least $1 - N e^{-c \log^2 k} - e^{-c \log^2 N}$ over $X$ and $W$.
\end{lemma}
\begin{proof}
The thesis follows from Theorem 3.1 of \cite{bombari2022memorization}. Notice that our assumptions on the data distribution $P_X$ are stronger, and that our initialization of the very last layer (which differs from their Gaussian initialization) does not change the result. Our assumption $k = \bigO{d}$ respects the \textit{loose pyramidal topology} condition, and we left the overparameterization assumption unchanged. A main difference is that we do not assume the activation function $\phi$ to be Lipschitz anymore. This, however, stop being a necessary assumption since we are working with a 2-layer neural network, and $\phi$ doesn't appear in the expression of Neural Tangent Kernel.
\end{proof}

\begin{lemma}\label{lemma:maxofsubGntk}
We have that
\begin{equation}
    \max_i \left|  \phi'(W x_i)^\top \phi'(Wz) \right| = \bigO{\sqrt k \log k },
\end{equation}
with probability at least $1 - \exp(-c \log^2 k)$ over $X$ and $W$.
\end{lemma}

\begin{proof}
    Let's look at the term $i = 1$, and let's call $x := x_1$ for compactness.
    As in Lemma \ref{lemma:Gapprox5}, we define $z'$ as follows
    \begin{equation}
        z' = \frac{ \sqrt d \left( I - \frac{xx^\top}{d} \right) z }{\norm{ \left( I - \frac{xx^\top}{d} \right) z}_2}.
    \end{equation}

    We have $x^\top z' = 0$, and $\norm{z'}_2 = \sqrt d$. Furthermore, by \eqref{eq:normissubG}, we also have that $\norm{z - z'}_2$ is a sub-Gaussian random variable, in the probability space of $x$. 
    It will also be convenient to characterize $\norm{\phi'(Wz')}_2$. We have, indicating with $w_j$ the $j$-th row of $W$,
    \begin{equation}
    \begin{aligned}
        \norm{\phi'(Wz')}_2^2 &= \sum_{j=1}^k \phi'^2(w_j^\top z') =  \sum_{j=1}^k \E_W \left[ \phi'^2(w_j^\top z') \right] + \sum_{j=1}^k \left( \phi'^2(w_j^\top z') - \E_W \left[ \phi'^2(w_j^\top z') \right] \right) \\
        &= k \E_\rho \left[ \phi'^2(\rho) \right] + \sum_{j=1}^k \left( \phi'^2(\rho_j) - \E_\rho \left[ \phi'^2(\rho) \right] \right),
    \end{aligned}
    \end{equation}
    where the $\rho$ and $\rho_j$'s indicate independent standard Gaussian random variables. Since $\phi'$ is Lipschitz and non-zero, we have that $\E_\rho \left[ \phi'^2(\rho) \right] = \Theta(1)$. Furthermore, by Bernstein inequality (cf. Theorem 2.8.1 in \cite{vershynin2018high}), we have $\sum_{j=1}^k \left( \phi'^2(\rho_j) - \E_\rho \left[ \phi'^2(\rho) \right] \right) = \bigO{k}$ with probability at least $1 - 2 \exp (-c k)$. This means that, with this probability (which is over $W$), we have
    \begin{equation}\label{eq:normonz}
        \norm{\phi'(Wz')}_2 \leq C \sqrt k,
    \end{equation}
    where $C$ is a numerical constant. Notice that the same result holds for $\norm{\phi'(Wx)}_2$. We also have $\opnorm{W} = \bigO{1}$, which happens with probability at least $1 - \exp(-ck)$ over $W$, because of Theorem 4.4.5 of \cite{vershynin2018high} and Assumption \ref{ass:dlargentk}.
    
    Since $\phi'$ is Lipschitz, we can write
    \begin{equation}\label{eq:upperboundlemmantk}
        \left| \phi'(W x)^\top \phi'(Wz) \right| \leq \left| \phi'(W x)^\top \phi'(Wz') \right| + L \norm{\phi'(W x)}_2 \opnorm{W} \norm{z -z'}_2.
    \end{equation}

    Let's look at the first term of the LHS. Since $x^\top z' = 0$, and $\norm{x}_2 = \norm{z'}_2 = \sqrt d$, we have
    \begin{equation}
        \phi'(W x)^\top \phi'(Wz') = \phi'(\hat \rho_a)^\top \phi'(\hat \rho_b),
    \end{equation}
    where $\hat \rho_a$ and $\hat \rho_b$ are two independent Gaussian vectors,  
    each of them with independent entries. As $\phi$ is an even function (see Assumption \ref{ass:activationfuncntk}), we have that $\phi'(\hat \rho_a)$ and $\phi'(\hat \rho_b)$ are independent sub-Gaussian vectors, with norm upper bounded by a numerical constant. 
    Thus,
    \begin{equation}
        \P \left( \left| \phi'(W x)^\top \phi'(Wz') \right| > C \sqrt k \log k \right) < 2 \exp(-c_1 \log^2 k),
    \end{equation}
    where this holds due to \eqref{eq:normonz}.
    Let's now look at the second term of the LHS of \eqref{eq:upperboundlemmantk}. We have
    \begin{equation}
        \norm{\phi'(W x)}_2 \opnorm{W} \norm{z -z'}_2 \leq C \sqrt k C_1 \norm{z -z'}_2 = \bigO{\log k \sqrt k},
    \end{equation}
    where the inequality holds with probability at least $1 - \exp(-ck)$ over $W$, and the equality holds with probability at least $1 - 2 \exp(-c_2 \log^2 k)$ over $x$.
    Taking the intersection between the previously described high probability events, from \eqref{eq:upperboundlemmantk} we get
    \begin{equation}
        \left| \phi'(W x)^\top \phi'(Wz) \right| = \bigO{\sqrt k \log k},
    \end{equation}
    with probability at least $1 - 2 \exp(-c \log^2 k)$ over $x$ and $W$.
    Thus, performing a union bound over the different $x_i$'s, we get the thesis.
\end{proof}

\begin{lemma}\label{lemma:opnormIntk}
\begin{equation}
    \opnorm{\nabla_z \Phi(z)^\top \Phi(X)^\top} = \bigO{\log k \left( \sqrt{k} + \sqrt{N} \right) \sqrt d},
\end{equation}
with probability at least $1 - \exp(-c \log^2 k)$ over $X$, $W$ and $z$.
\end{lemma}
\begin{proof}
    We have, by application of the chain rule,
    \begin{equation}
    \begin{aligned}
        \nabla_z \Phi(z)^\top \Phi(X)^\top &= 2 \, \nabla_z \Phi'(z)^\top \Phi'(X)^\top \\
        &= 2\, \nabla_z \left( z \otimes \phi'(Wz) \right)^\top \left( X * \phi' (XW^\top) \right)^\top \\
        &= 2 \left(I * \mathbf 1  \phi'^\top(Wz) + \mathbf 1 z^\top * W^\top \textup{diag} \left( \phi'' (Wz) \right)\right) \left( X * \phi' (XW^\top) \right)^\top \\
        &= 2\, X^\top \circ \mathbf 1 \left( \phi'(XW^\top) \phi'(Wz) \right)^\top + 2 \,\mathbf 1 \left( X z \right)^\top \circ W^\top \textup{diag} \left( \phi'' (Wz) \right) \phi' (XW^\top) ^\top\\
        &= 2 \, X^\top \textup{diag} \left( \phi'(XW^\top) \phi'(Wz) \right) + 2 \, W^\top \textup{diag} \left( \phi'' (Wz) \right)\phi' (XW^\top)^\top \textup{diag} \left( Xz \right).
    \end{aligned}
    \end{equation}
    Let's now condition on the event $\opnorm{W} = \bigO{1}$, which happens with probability at least $1 - \exp(-cd)$ over $W$, because of Theorem 4.4.5 of \cite{vershynin2018high} and Assumption \ref{ass:dlargentk}.
    
    Let's now look at the two terms separately. For the first, we have
    \begin{equation}
    \begin{aligned}
        \opnorm{X^\top \textup{diag} \left( \phi'(XW^\top) \phi'(Wz) \right) } &\leq \opnorm{X} \max_i \left|  \phi'(W x_i)^\top \phi'(Wz) \right|  \\
        &= \bigO{\log k \sqrt k \left( \sqrt {N} + \sqrt {d}\right)},
    \end{aligned}
    \end{equation}
    where the last equality is true because of Lemmas \ref{lemma:matiidrows} and \ref{lemma:maxofsubGntk} and holds with probability at least $1 - \exp(-c d) - \exp(-c \log^2 k)$ over $X$ and $W$.
    For the other term, we have
    \begin{equation}
    \begin{aligned}
        \opnorm{W^\top \textup{diag} \left( \phi'' (Wz) \right) \phi' (XW^\top) \textup{diag} \left( Xz \right)  } &\leq  \opnorm{W} L \opnorm{\phi' (XW^\top)}  \max_i \left|  x_i^\top z \right|\\
        &= \bigO{\log k \sqrt d \left( \sqrt {N} + \sqrt {k}\right) }.
    \end{aligned}
    \end{equation}
    The last step uses Lemma \ref{lemma:matiidrows} and the fact that we are taking the maximum over $k$ sub-Gaussian random variables such that $\subGnorm{x_i^\top z} \leq C \norm{z}_2 = C \sqrt d$, and it holds with probability at least $1 - 2 \exp(-c k) - \exp(-c \log^2 k)$ over $X$.

    Putting everything together, and using Assumption \ref{ass:dlargentk}, we get
    \begin{equation}
    \begin{aligned}
        \opnorm{\nabla_z \Phi(z)^\top \Phi(X)^\top} &= \bigO{\log k \sqrt k \left( \sqrt {N} + \sqrt {d}\right)} + \bigO{\log k \sqrt d \left( \sqrt {N} + \sqrt {k}\right) } \\
        &= \bigO{\log k \left( \sqrt{k} + \sqrt{N} \right) \sqrt d},
    \end{aligned}
    \end{equation}
    which gives the thesis.
\end{proof}

\subsection{Proof of Theorem \ref{thm:ntk}}\label{proof:ntk}
\begin{proof}[Proof of Theorem \ref{thm:ntk}]
    We have
    \begin{equation}
    \begin{aligned}
        \mathcal S(z) &\leq \sqrt d \opnorm{\nabla_z \Phi(z)^\top \Phi(X)^\top} \opnorm{K^{-1}} \norm{Y}_2\\
        &= \sqrt d \bigO{\log k \left( \sqrt{k} + \sqrt{N} \right) \sqrt d} \bigO{\frac{1}{dk}} \bigO{\sqrt N}\\
        &= \bigO{\log k \left( 1 + \sqrt{N / k} \right) \frac{\sqrt N}{\sqrt k}},
    \end{aligned}
    \end{equation}
    where the second line is justified by Lemmas \ref{lemma:evminntk} and \ref{lemma:opnormIntk}, and holds with probability at least $1 - N e^{-c \log^2 k} - e^{-c \log^2 N}$ over $X$ and $W$. The second statement is immediate.
\end{proof}

\section{Additional Experiments}\label{app:exp}
In this section, we provide additional experimental results for the activation functions $\phi(x) = \cos(x)$ (which is even) and $\phi(x) = \frac{e^x}{1 + e^x}$ (namely, softplus or smooth ReLU, which is not even). Again, our theoretical findings are supported by numerical evidence both on the RF (see Figure \ref{fig:rfnewact}) and the NTK regression model (see Figure \ref{fig:ntknewact}). 

\begin{figure}[!b]
  \begin{center}
    \includegraphics[width=0.99\textwidth]{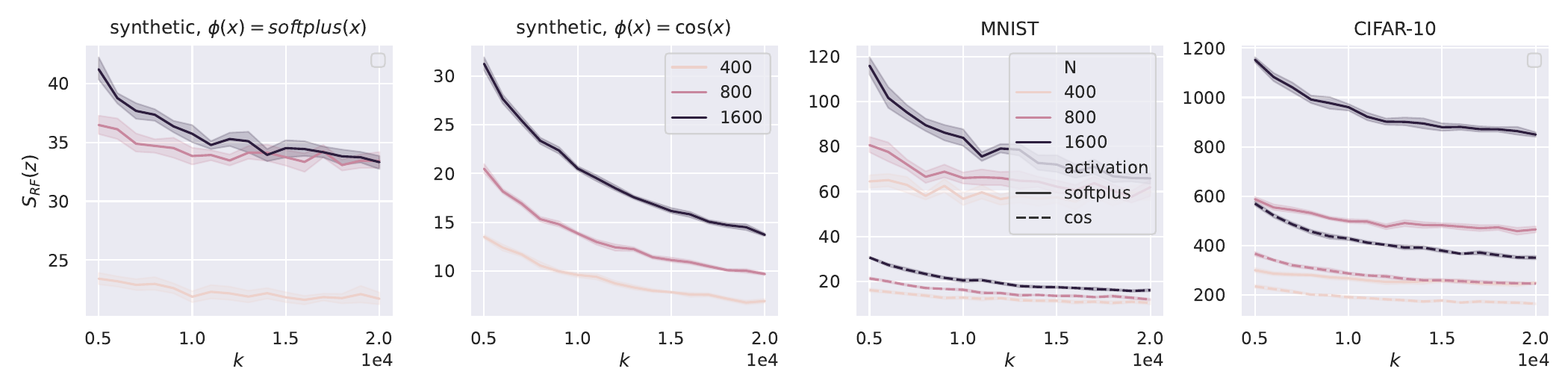}
  \end{center}
  \caption{Sensitivity as a function of the number of parameters $k$, for an RF model with synthetic data sampled from a Gaussian distribution with input dimension $d=1000$ (first and second plot), and with inputs taken from two classes of the MNIST and CIFAR-10 datasets (third and fourth plot respectively). Different curves correspond to different values of the sample size $N\in \{400, 800, 1600\}$ and to different activations ($\phi(x) = \textup{softplus}(x)$ or $\phi(x) = \cos(x)$). We plot the average over $10$ independent trials and the confidence interval at $1$ standard deviation. The values of the sensitivity are significantly larger for $\phi(x) = \textup{softplus}(x)$ than for $\phi(x) = \cos(x)$.}
  \label{fig:rfnewact}
\end{figure}

\clearpage

\begin{figure}[!t]
  \begin{center}
    \includegraphics[width=0.99\textwidth]{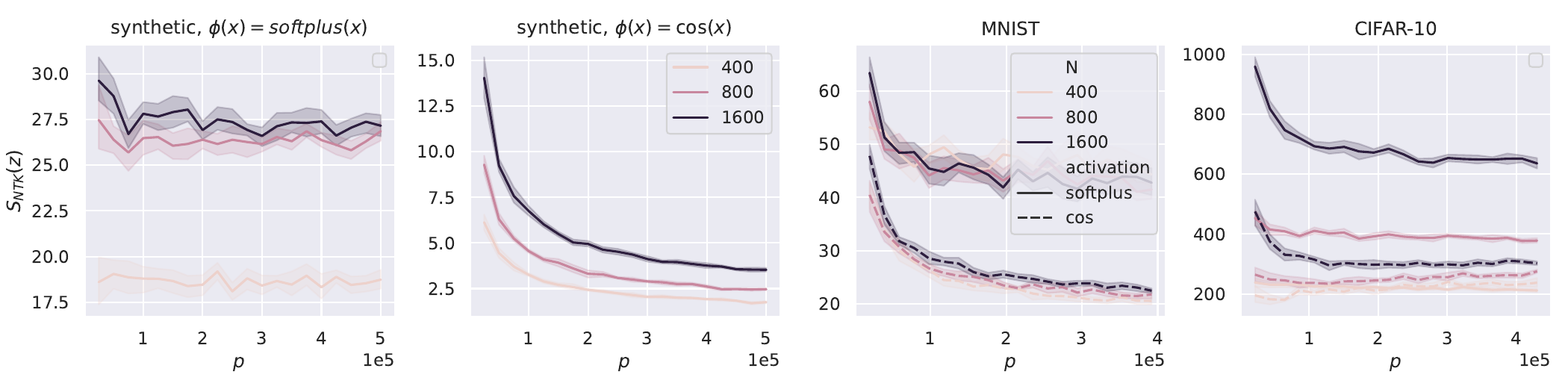}
  \end{center}
  \caption{Sensitivity as a function of the number of parameters $p$, for an NTK model with synthetic data sampled from a Gaussian distribution with input dimension $d=1000$ (first and second plot), and with inputs taken from two classes of the MNIST and CIFAR-10 datasets (third and fourth plot respectively). Different curves correspond to different values of the sample size $N\in \{400, 800, 1600\}$ and to different activations ($\phi(x) = \textup{softplus}(x)$ or $\phi(x) = \cos(x)$). We plot the average over $10$ independent trials and the confidence interval at $1$ standard deviation. The values of the sensitivity are significantly larger for $\phi(x) = \textup{softplus}(x)$ than for $\phi(x) = \cos(x)$.}
  \label{fig:ntknewact}
\end{figure}

The code to reproduce the experiments and all the plots in the paper is available at the public GitHub repository \href{https://github.com/simone-bombari/beyond-universal-robustness}{\texttt{https://github.com/simone-bombari/beyond-universal-robustness}}.

\end{document}